%% file: main.tex
\documentclass[11pt]{article}

\input{header.tex}

\newcommand{\sivan}[1]{\todo[size=\tiny, color=lightgray]{sivan: #1}}

\newcommand{\chavdar}[1]{\todo[size=\tiny, color=brown]{chavdar: #1}}

\begin{document}

\title{Simplicial covering dimension of extremal concept classes}

\author{
    Ari Blondal\thanks{McGill University, \texttt{ari.blondal@mail.mcgill.ca, hamed.hatami@mcgill.ca}. Hamed Hatami is supported by an NSERC grant.} \and Hamed Hatami\footnotemark[1] \and
    Pooya Hatami~\thanks{Ohio State University, \texttt{\{hatami.2, lalov.1, tretiak.2\}@osu.edu}} \and  Chavdar Lalov\footnotemark[2] \and Sivan Tretiak\footnotemark[2]
}

\maketitle

\begin{abstract}

Dimension theory is a branch of topology concerned with defining and analyzing dimensions of geometric and topological spaces in purely topological terms. In this work, we adapt the classical notion of topological dimension (Lebesgue covering) to binary concept classes. The topological space naturally associated with a concept class is its space of realizable distributions. The loss function and the class itself induce a simplicial structure on this space, with respect to which we define a simplicial covering dimension.

We prove that for finite concept classes, this simplicial covering dimension exactly characterizes the list replicability number (equivalently, global stability) in PAC learning. This connection allows us to apply tools from classical dimension theory to compute the exact list replicability number of the broad family of extremal concept classes.  
\end{abstract}

\tableofcontents

\section{Introduction}  
\label{sec:intro}

In recent years, several intriguing advances in learning theory have been made using tools from topology.  These developments point to a potential connection between learning theory and \emph{topological dimension theory}, a classical branch of topology pioneered by Brouwer, Lebesgue, and others in the early twentieth century, who aimed to formalize a general notion of dimension that demonstrates $\mathbb{R}^m$ and $\mathbb{R}^n$ are homeomorphic only when $m = n$.

Motivated by this connection, we introduce a new notion of dimension for concept classes, inspired by the classical Lebesgue covering dimension and the simplicial structures that naturally arise in learning problems.  We focus on binary classification, the task of learning an unknown function that maps the elements of a domain $\cX$ to one of two possible labels, typically denoted by $\pm 1$.  Formally, a binary classification problem is specified as a \emph{binary concept class} $\mathcal{C}$, a set of functions from a domain $\mathcal{X}$ to the label set $\set{\pm 1}$. Throughout this work, we restrict attention to finite domains $\cX$.

We define the \emph{simplicial covering dimension} $\SCdim(\mathcal{C})$ of a concept class by studying the topology of its space of realizable distributions, endowed with the simplicial structure induced by $\mathcal{C}$ and the loss function. Our definition is partly inspired by the line of work~\cite{chase2023replicabilitystabilitylearning, vander2024replicability, DPVV23, chase2023local, BGHH2025stabilitylistreplicabilityagnosticlearners, chornomaz2025spherical}, which showed that  \emph{list-replicable learning} is intrinsically connected to the geometry and topology of the space of realizable distributions. 

In the finite-domain setting, the simplicial covering dimension turns out (\cref{thm:intro_SCD_LR_equivalence}) to coincide with the \emph{list-replicability number} of the class:
\begin{equation}
\label{eq:LCdim_eqls_LR}
\SCdim(\cC) = \LR(\cC) - 1,
\end{equation}
where $\LR(\cC)$ denotes the list-replicability number of $\cC$. This well-studied parameter arose in recent years as part of a broader attempt to formalize the notion of replicability ~\cite{BLM20,malliaris2022unstable, chase2023replicabilitystabilitylearning, bun2023stability, karbasi2023replicability, esfandiari2023replicable, Esfandiarietal23, moran2023bayesian, eaton2024replicable, kalavasis2024replicable, kalavasis2023statistical}, which refers to the requirement that an algorithm produce consistent outcomes when repeated under similar conditions and with similar data. Our framework allows us to import classical tools from dimension theory into the analysis of replicable learning.

In particular, by analyzing $\SCdim$, we determine the list-replicability number of \emph{extremal concept classes} (also known as ample concept classes), a family that includes many of the most natural examples of concept classes. Our main theorem (\Cref{thm:intro_LCdim_extremal})  establishes that every \emph{extremal} concept class $\cE$ over a finite domain $\cX$, satisfies 
\begin{equation}
\label{eq:main_result}
        \SCdim(\cE) = 
        \begin{cases}
            \VCdim(\cE)-1 & \text{ if }\cE=\set{\pm 1}^\cX  \\
            \VCdim(\cE) &\text{otherwise} 
        \end{cases}, 
\end{equation}
where $ \VCdim(\cE)$ denotes the VC dimension of $\cE$. 

This theorem resolves the list-replicability number of several well-studied classes, such as axis-parallel boxes and downward-closed classes~\cite[Section 3.2]{chalopin2022unlabeled}, whose values were previously unknown. Beyond the new cases, it also unifies the characterization of the list replicability number for several previously known cases, such as the binary cube~\cite{chase2023replicabilitystabilitylearning}, threshold functions~\cite{chase2023replicabilitystabilitylearning}, and more generally halfspaces~\cite{blondal2025borsukulamreplicablelearninglargemargin}.

\subsection{Preliminaries and definitions}
 
In this section, we provide precise definitions of the notions discussed in the introduction.
\subsubsection{PAC learning, VC dimension, and list replicability}

In probably approximately correct (PAC) learning, the learner is given parameters $\delta,\epsilon>0$ and receives training data $S$ consisting of $n=n(\cC,\delta,\epsilon)$ independent labeled examples drawn from an unknown but fixed distribution $\mu$ over $\mathcal{X} \times \set{\pm 1}$.  We work in the \emph{realizable} setting: there exists some concept $c \in \mathcal{C}$ that correctly labels all examples in the support of $\mu$. The learner’s task is to use the training data $S$ to output, with probability at least $1-\delta$, a \emph{hypothesis} $h:\cX \to \set{\pm 1}$ whose \emph{population loss}
\[\loss_\mu(h) \coloneqq \Pr_{(x,b) \sim \mu}[h(x) \neq b]\]
is at most $\epsilon$. 

Throughout this work, a \emph{learning rule} refers to a (randomized) function $\bm{\cA}$ that maps any sample
$S \in \bigcup_{n=0}^\infty (\cX \times \Set{\pm 1})^n$ 
to a hypothesis $\bm{\cA}(S) \in \Set{\pm 1}^{\cX}$. Since our primary focus is sample complexity rather than computational efficiency, we impose no computability constraints on $\bm{\cA}$.

\paragraph{VC dimension.}  A  concept class $\cC \subseteq \set{\pm 1}^\cX$ shatters a set $S \subseteq \cX$ if $\set{ c|_S :~c \in \cC } = \set{\pm 1}^S$, where $c|_S$ denotes the restriction of $c$ to $S$.  The  \emph{Vapnik–Chervonenkis} (VC) dimension of $\cC$ is defined as 
\[ \VCdim(\cC) \coloneqq \sup \set{ |S|:~S\subseteq \cX \text{ is shattered by }\cC}. \] 
The fundamental theorem of PAC learning states that the sample complexity of PAC learning is determined by the VC dimension. More precisely, the optimal sample size for PAC learning a binary class $\cC$ with accuracy parameter $\epsilon$ and confidence parameter $\delta$ is $\Theta_{\epsilon,\delta}(\VCdim(\cC))$. 
 
\paragraph{List replicability.} List replicability, introduced in~\cite{chase2023replicabilitystabilitylearning,DPVV23}, is an elegant reformulation of global stability~\cite[Definition 1]{chase2023replicabilitystabilitylearning}.

\begin{definition}[List replicability]
\label{def:list}
The \emph{list replicability number} of a concept class $\cC \subseteq \set{\pm 1}^\cX$, denoted $\LR(\cC)$, is the smallest integer $L$ such that the following holds. For every $\epsilon, \delta > 0$, there exists a sample size $n(\cC,\epsilon, \delta)$ and a learning rule $\bm{\cA}$ such that:

For every realizable distribution $\mu$ on $\cX \times \set{\pm 1}$, there exists a list of hypotheses $h_1, \dots, h_L \in \set{\pm 1}^\cX$ satisfying  
\begin{spreadlines}{1em}
    \begin{align*}
        &\bullet\ \loss_\mu(h_i) \le \epsilon \text{ for all } i=1,\ldots,L;\\
        &\bullet\ \Pr_{S \sim \mu^n} [\bm{\cA}(S) \not\in \Set{h_1,\ldots,h_L}] \le  \delta \ \text{ where } n=n(\epsilon,\delta).
    \end{align*}
\end{spreadlines}

\end{definition}

\subsubsection{Extremal classes} 

A class $\cC$ over a domain $\cX$ \emph{strongly shatters} a set $S \subseteq \cX$ 
if there exists a labeling $a \in \set{\pm 1}^{\cX \setminus S}$ such that 
\[
\set{ c|_S :~c \in \cC,\ c|_{\cX \setminus S}=a } = \set{\pm 1}^S.
\]
In other words, $\cC$ realizes all $2^{|S|}$ labelings on $S$ while fixing the labels outside $S$ in some prescribed way. A class $\cC$ is called \emph{extremal} if every set shattered by $\cC$ is also strongly shattered.

Many natural concept classes are either extremal or admit natural extremal extensions. 
We refer the reader to \cite[Section 3.2]{chalopin2022unlabeled} for a more comprehensive list of known examples; here, we present a few illustrative ones.  
\begin{example}[Sign patterns of convex sets~\cite{MR683734}]
\label{ex:convex}
Let $K \subset \mathbb{R}^n$ be a convex set. The class 
\[
C(K) \coloneqq \set{ \sign(v):~v \in K,\ v_i \neq 0 \ \forall i \le n },
\]
where $\sign(v) \in \set{\pm 1}^n$ denotes the sign pattern of $v$, 
is extremal.
\end{example}

\begin{example}[Homogeneous half-spaces]
\label{ex:point_hom_half_space}
Let $P \subset \mathbb{R}^d$ be a finite set of points.
Define the concept class of homogeneous half-spaces $\cH_P$ over the domain $P$ as follows:
each homogeneous half-space in $\mathbb{R}^d$ whose defining hyperplane avoids $P$ induces a labeling of $P$ by $\set{\pm 1}$.
Explicitly, the label of $p \in P$ is given by
\[
p \mapsto \sign(\inp{v}{p}),
\]
where the inequality $\inp{v}{x} > 0$ defines the half-space.

The set of maps $p \mapsto \inp{v}{p}$ (for all $v \in \mathbb{R}^d$) is convex, and therefore, by \Cref{ex:convex}, the class $\cH_P$ is extremal.
\end{example}

\begin{example}[Axis-parallel boxes]
Let $P \subset \mathbb{R}^d$ be a finite set of points, no two of which share a coordinate.  Each axis-parallel box in $\mathbb{R}^d$ induces a labeling of $P$ by $\pm 1$,  
depending on whether a point lies inside or outside the box.  
The corresponding concept class is extremal. 
\end{example}

\begin{example}[Median classes] 
A class $\cC \subseteq \set{\pm 1}^\cX$ is called \emph{median} if it is closed under taking the majority of three concepts.  
More precisely, for any $c_1,c_2,c_3 \in \cC$, the concept  
\[
x \mapsto \mathrm{maj}(c_1(x),c_2(x),c_3(x))
\] 
also belongs to $\cC$.  
All median classes are extremal~\cite[Proposition~2]{MR2215426}. 
\end{example}

While not every class is extremal, a well-known open problem asks whether every class can be extended to an extremal one without significantly increasing its VC dimension.
 
\begin{question}[\cite{moran2016labeled,chase2024dual}]
\label{q:extremal}
    Does there exist a function $t:\mathbb{N} \to \mathbb{N}$ such that,  
    for every binary class $\cC$ over a finite domain $\cX$, there is an extremal class 
    $\cE \supseteq \cC$ over $\cX$ with
    \[
    \VCdim(\cE) \le t(\VCdim(\cC))?
    \]
\end{question}

\subsubsection{Topological dimension theory}

For two families $\cA = \set{A_i}_{i \in I}$ and $\cB = \set{B_i}_{i \in I}$ of subsets of a set $X$, indexed by the same set $I$, we say that $\cB$ is a \emph{shrinkage} of $\cA$ if $B_i \subseteq A_i$ for all $i \in I$.

More generally, a family $\cB = \set{B_j}_{j \in J}$ of subsets of $X$ is said to be a \emph{refinement} of another family $\cA = \set{A_i}_{i \in I}$, written $\cA \prec \cB$, if for every $B_j$, there exists $A_i$ such that $B_j \subseteq A_i$.
Every shrinkage of $\cA$ is, in particular, a refinement of $\cA$.

The \emph{order} of a family $\cA=\set{A_i}_{i \in I}$ of sets is defined as 
\[ \ord(\cA) \coloneqq \max_{x} (|\set{i \in I : \ x \in A_i}|-1),\] 
i.e., one less than the maximum number of sets in $\cA$ that have a nonempty intersection.

Now suppose $\cA$ is a \emph{finite open cover} of a topological space $X$. The \emph{Lebesgue covering number} of $\cA$ is defined\footnote{Traditionally, $\LC(\cA)$ is defined by taking the minimum of $\ord(\cB)$ over all finite open covers $\cB$ that refine $\cA$. However, as shown in~\cite[Proposition 1.1.7]{MR3242807}, it suffices to consider only shrinkages of $\cA$.} as 
\[\LC(\cA) \coloneqq \min \set{\ord(\cB):~ \text{$\cB$ is an open cover of $X$ and a shrinkage of $\cA$}}.  \]

\begin{definition}[Topological dimension]
\label{def:lebesgue}
The Lebesgue covering dimension (or simply the topological dimension) of $X$ is defined by
\begin{equation}
\label{eq:LCdim}
\dim(X)  \coloneqq \sup \Set{ \LC(\cA) :~\cA \text{ is a finite open cover of } X }.
\end{equation}
\end{definition}

\begin{remark}
\label{rem:outside_X}
In this paper, we will always work with topological spaces of the form $X \subseteq \mathbb{R}^n$ endowed with the subspace topology inherited from $\mathbb{R}^n$. Thus, a set $S \subseteq X$ is open if and only if $S = U \cap X$ for some open set $U \subseteq \mathbb{R}^n$. For such spaces, it is often convenient to allow covers of $X$ whose elements may extend outside $X$. In this case, we identify $\cA = \set{A_i}_{i \in I}$ with $\set{A_i \cap X}_{i \in I}$, and call $\cA$ an open cover of $X$ if  $\set{A_i \cap X}_{i \in I}$ is an open cover of $X$. We write $\cA \prec \cB$ to mean $\set{A_i \cap X}_{i \in I} \prec \set{B_j \cap X}_{j \in J}$. However, we still calculate $\ord(\cA)$ as one less than the maximum number of sets containing any one point $x \in X$, and define $\LC(\cA)$ accordingly.

\end{remark}

The notion of Lebesgue covering dimension of \Cref{def:lebesgue} was formally introduced by Eduard \v{C}ech in 1933, building on a theorem of Lebesgue that serves as one of the main tools in this paper.\footnote{The Lebesgue covering theorem is usually stated for closed sets, as in \cref{thm:lebesgue_covering_thm}. The open-set formulation presented here is derived from the closed version in \Cref{sec:basic_facts}.}

\begin{restatable}[Lebesgue covering theorem]{theorem}{lebesguecoveringtheorem}\label{thm:lebesgue}
    Suppose a $d$-dimensional cube $[-1,1]^d \subset \R^d$  is covered by a finite family $\cA$ of open sets, none of which contains points of opposite faces of the cube.
    Then $\ord(\cA) \geq d$.
\end{restatable}

See~\Cref{fig:lebesgue_covering_dimension} for an illustration of the Lebesgue covering dimension of a square in $\mathbb{R}^2$.

\input{fig_lebesgue}

We also rely on the following standard facts about topological dimension:

\begin{theorem}[{\cite[Theorems 50.2 and 50.6]{munkres2000topology}}]\label{theorem:LC_of_unions}  Topological dimension satisfies:
\begin{itemize}
\item[(i)] Every compact subspace of $\R^d$ has topological dimension at most $d$.
\item[(ii)] Let $X = Y \cup Z$, where $Y$ and $Z$ are closed subspaces of $X$ having finite topological dimension. Then
    \[
        \dim(X) = \max \Set{ \dim(Y), \dim(Z)}.
    \]
\end{itemize}    
\end{theorem}

From \Cref{thm:lebesgue}, the cube $[-1,1]^d$ has topological dimension at least $d$, while \Cref{theorem:LC_of_unions}~(i) provides the upper bound of $d$. Hence, its topological dimension is exactly $d$. Consequently, any set homeomorphic to $[-1,1]^d$, such as a $d$-simplex, or the unit ball in $\mathbb{R}^d$, also has topological dimension $d$.  

For further reading on dimension theory, see the classical texts~\cite{dimensiontheory1941,MR482697}.

\subsubsection{Simplicial covering dimension of binary concept classes}

In this section, we describe how the notion of topological dimension can be naturally adapted to the space of realizable distributions associated with a concept class. Our starting point is to represent every realizable distribution as a point in $\mathbb{R}^\cX$.
Every vector $\mu$ on the $\ell_1$-unit sphere in $\R^\cX$ naturally encodes a distribution over $\cX \times \set{\pm 1}$: for each $x \in \supp(\mu)$, the distribution assigns mass $|\mu(x)|$ to the labeled example $(x,\sign(\mu(x)))$. This correspondence identifies the realizable distributions of a concept class $\cC \subseteq \set{\pm 1}^\cX$ with the subset
\begin{equation}
\label{eq:definition_D_C}\Delta_\cC \coloneqq \Set{\mu \in \mathbb{R}^\cX :~\norm{\mu}_1=1 \text{ and } \exists c \in \cC \text{ with } c(x)=\sign(\mu(x)) \ \forall x \in \supp(\mu)} \subseteq \R^\cX.
\end{equation}

Note, however, that the topology of $\mathbb{R}^\cX$ does not capture the connection between $\Delta_\cC$ and the binary classification task, since the relevant notion of proximity in learning is not the Euclidean distance but rather the population loss, which only measures discrepancies on points where the given $\pm 1$ labels differ.   For instance, a given hypothesis $h$ incurs zero population loss on all the realizable distributions in 
\begin{equation}\label{eq:zero-loss_set}
B_h \coloneqq \set{\mu \in \R^\cX :~\norm{\mu}_1 = 1 \text{ and } \loss_\mu(h) = 0},
\end{equation}
even though these distributions may be far apart geometrically in $\mathbb{R}^\cX$. We refer to $B_h$ as the \emph{zero-loss set} around $h$.

Each $B_h$ is a $(|\cX|-1)$-simplex in $\mathbb{R}^\cX$, and $\Delta_\cC = \bigcup_{c \in \cC} B_c$. Hence, unless $\cC$ is empty, $\Delta_\cC$, equipped with the Euclidean topology of $\mathbb{R}^\cX$, has the same topological dimension as each simplex $B_c$, namely $|\cX|-1$. However, from the perspective of population loss, it is natural to regard all distributions within a zero-loss set $B_c$ as being at “distance” zero from one another, since they all induce the same labeling of $\cX$. Consequently, each $B_c$ should be viewed as a $0$-dimensional set in this context.
 
In other words, a suitable notion of topological dimension for $\Delta_\cC$ should consider only those covers  $\mathcal{A}$ in \eqref{eq:LCdim} that respect the equivalence relation induced by each zero-loss set $B_h$.  We formalize this idea in the following definition.

\begin{definition}[Simplicial covering dimension of a binary concept class]  
\label{def:TD_of_binary}
Let $\cC \subseteq \set{\pm 1}^\cX$ be a binary concept class over a finite domain $\cX$, and let $\cB = \set{B_h}_{h \in \set{\pm 1}^\cX}$ be the cover of $\Delta_\cC \subseteq \mathbb{R}^\cX$ defined above. The \emph{simplicial covering dimension} of $\cC$ is
\[\SCdim(\cC) \coloneqq  \sup \Set{ \LC(\cA) :~\cA \text{ is a finite open cover of } \Delta_\cC \text{ satisfying } \cA \prec \cB}. \]
 
\end{definition}

The definition of the simplicial covering dimension for concept classes can be naturally interpreted through the lens of simplicial complexes.

Consider a finite geometric simplicial complex $\Gamma$ in $\mathbb{R}^n$ (\Cref{def:geomSimp}). By definition, $\Gamma$ is a collection of simplices whose union forms the polyhedron $\norm{\Gamma} = \bigcup_{\sigma \in \Gamma} \sigma \subseteq \mathbb{R}^n$. It is natural to study the topological dimension of $\norm{\Gamma}$ relative to this simplicial cover. More generally, for simplicial complexes $\Gamma \subseteq \Gamma'$, the faces $\sigma \in \Gamma'$ need not lie entirely in $\norm{\Gamma}$, but by our convention~\Cref{rem:outside_X} they nevertheless define a cover of $\norm{\Gamma}$.  

\begin{definition}
Let $\Gamma \subseteq \Gamma'$ be geometric simplicial complexes in $\mathbb{R}^n$.  
The \emph{relative simplicial covering dimension} of $\Gamma$ with respect to $\Gamma'$ is defined by  
\[\SCdim_{\Gamma'}(\Gamma) \coloneqq  \sup_\cA \LC(\cA),  \] 
where the supremum is taken over all \emph{finite open covers} $\cA$ of $\norm{\Gamma}$ satisfying $\cA \prec \Gamma'$.
\end{definition} 

The quantity $\SCdim_{\Gamma'}(\Gamma)$ only depends\footnote{See \Cref{prop:abstract_Simp}.} on the abstract simplicial structures (\Cref{def:AbsSimp}) of $\Gamma$ and $\Gamma'$, which allows us to define:

\begin{definition}
\label{def:LC_simplicial}
Let $K \subseteq K'$ be finite abstract simplicial complexes. The relative simplicial covering dimension of $K$ with respect to the faces of $K'$ is  defined by 
\[\SCdim_{K'}(K) \coloneqq \SCdim_{\Gamma'}(\Gamma), \] 
where $\Gamma'$ is any geometric realization of $K'$ and $\Gamma \subseteq \Gamma'$ is its restriction to $K$. 
\end{definition}

Understanding the quantities $\SCdim_K(K)$, and more generally $\SCdim_{K'}(K)$, appears to be an intriguing problem in its own right, independent of the learning-theoretic motivations of this work.

\begin{question}
\label{q:simplicalLC}
Is there a simple combinatorial characterization of  $\SCdim_K(K)$, and more generally $\SCdim_{K'}(K)$, where $K' \supseteq K$ are finite simplicial complexes?
\end{question}

Finally, the simplicial covering dimension of concept classes introduced in \Cref{def:TD_of_binary} arises as a special case of \Cref{def:LC_simplicial}.
Indeed, since  $\Delta_\cC = \bigcup_{c \in \cC}  B_c$, the set $\Delta_\cC$ can be viewed as the polyhedron of the geometric simplicial complex $\Gamma$ whose facets are the simplices $B_c$ for $c \in \cC$.
Letting $\Gamma'$ denote the simplicial complex whose facets are $B_h$ for all $h \in \set{\pm 1}^\cX$, we recover \Cref{def:TD_of_binary} from \Cref{def:LC_simplicial}.

\subsection{Main contributions}

Our first theorem establishes that, for binary concept classes over finite domains, the simplicial covering dimension exactly characterizes the list replicability number.

\begin{atheorem}\label{thm:intro_SCD_LR_equivalence}
Let $\cC \subseteq \set{\pm 1}^\cX$ be a concept class over a finite domain $\cX$. Then
\[
\LR(\cC) = \SCdim(\cC) + 1.
\]
\end{atheorem}

Combining \Cref{thm:intro_SCD_LR_equivalence} with the known value of $\LR(\set{\pm 1}^\cX)$ from~\cite{chase2023replicabilitystabilitylearning} yields the following.

\begin{theorem}
\label{thm:cube}
For any  finite domain $\cX$, we have
\[\SCdim(\set{\pm 1}^\cX)=|\cX|-1 \text{ and } \LR(\set{\pm 1}^\cX)=|\cX|.\]
Consequently, every binary class $\cC$ over a finite domain satisfies  
\[\SCdim(\cC) \ge \VCdim(\cC)-1 \ \text{ and } \ \LR(\cC) \ge \VCdim(\cC).\] 
\end{theorem}
\begin{proof}
The equality $\LR(\set{\pm 1}^\cX)=|\cX|$ is due to~\cite{chase2023replicabilitystabilitylearning}. Combined with \Cref{thm:intro_SCD_LR_equivalence}, we obtain $\SCdim(\cC)=|\cX|-1$. Now, let $\cC \subseteq \set{\pm 1}^\cX$ be a class over a finite domain $\cX$, and let $S \subseteq \cX$ be a finite set. Trivially, we have  $\LR(\cC) \ge \LR(\cC|_S)$, and on the other hand, by considering any finite shattered set $S \subseteq \cX$, we have 
\[\SCdim(\cC)+1=\LR(\cC)\ge  \LR(\cC|_S)=\LR(\set{\pm 1}^S) = |S|. \qedhere \]
\end{proof}

\begin{remark}
In \cite{chase2023replicabilitystabilitylearning}, the lower bound $\LR(\set{\pm 1}^\cX) \ge |\cX|$ is proved using the Poincar\'e-Miranda theorem and the upper bound $\LR(\set{\pm 1}^\cX) \le |\cX|$ is proved by designing an explicit algorithm. In our framework, the upper bound is immediate, since $\SCdim(\cC)$ is by definition upper bounded by the topological dimension of $\Delta_{\cC}$, and $\Delta_{\set{\pm 1}^\cX}$ is the empty cross-polytope in $\mathbb{R}^\cX$, which has topological dimension $|\cX|-1$.
Moreover, as shown in \cref{lemma:extremal_lower_bound}, the lower bound can alternatively be recovered from the Lebesgue covering theorem.
\end{remark}

The difficulty of computing the list replicability number of a concept class is in part due to the markedly different techniques required for upper and lower bounds. Upper bounds typically use explicit learning algorithms engineered to output from a small list, whereas lower bounds rely on ad hoc topological arguments.

In light of \Cref{thm:intro_SCD_LR_equivalence}, a natural alternative is to work directly with the simplicial covering dimension of the class. Using this approach, we determine the exact list replicability number for all extremal classes.

\begin{atheorem}[Main theorem]\label{thm:intro_LCdim_extremal}
    If $\cE \neq \set{\pm 1}^\cX$ is an extremal class over a finite domain $\cX$, then
    \[
        \SCdim(\cE) =  \VCdim(\cE) \  \text{ and } \ \LR(\cE) =  \VCdim(\cE)+1. 
    \]
\end{atheorem}  

Note that the lower bound $\SCdim(\cE) \ge \VCdim(\cE)$ in \Cref{thm:intro_LCdim_extremal} improves by one upon the general lower bound $\SCdim(\cC) \ge \VCdim(\cC)-1$ from \Cref{thm:cube} that holds for all finite-domain concept classes.

\begin{proof}[Proof overview of \Cref{thm:intro_LCdim_extremal}]
The proof analyzes the simplicial structure of $\Delta_\cE$ and its relation to a cubical complex $\Gamma_\cE$ defined by the strongly shattered sets of $\cE$. The key property of extremal classes is that $\Gamma_\cE$ is contractible~\cite{chalopin2022unlabeled}. It was shown in~\cite[D.5]{chornomaz2025spherical} that this contractibility implies a deformation retraction of $\Delta_\cE$ onto $\Gamma_\cE$.

Here, we construct a different retraction tailored to our notion of dimension. This allows us to reduce the problem of computing $\SCdim(\cE)$ to analyzing $\Gamma_\cE$, where classical tools from dimension theory apply.   The simplicial covering dimension of $\Gamma_\cE$ is equal to the size of the largest strongly shattered subset of $\cX$, and matches the dimension required for the upper bound. To obtain the matching lower bound, we apply Lebesgue's covering theorem~\cite[Theorem IV 2]{dimensiontheory1941} to the maximal solid cube within $\Gamma_\cE$.

The difference between the cases $\cE = \set{\pm 1}^\cX$ and $\cE \neq \set{\pm 1}^\cX$ is essentially explained by the following observation. If $S \subsetneq \cX$ is strongly shattered by $\cE$, then $\Delta_\cE$ contains a full copy of the solid $|S|$-dimensional cube, as $\mu \in \Delta_\cE$ can allocate an arbitrary portion of its probability mass on coordinates outside $S$. In contrast, when $S = \cX$, the normalization condition $\norm{\mu}_1 = 1$ forces the entire mass to lie on $S$, so only the boundary of the $|S|$-dimensional cube is realized. 
\end{proof}

The extremal class of homogeneous half-spaces, discussed in \Cref{ex:point_hom_half_space}, naturally connects to the notion of \emph{sign-rank}, a key geometric measure of complexity in learning theory. The \emph{sign-rank} of a class $\cC \subseteq \set{\pm 1}^\cX$, denoted $\signrank(\cC)$, is the minimum $d$ such that $\cC$ can be realized as points and homogeneous half-spaces in $\R^d$. Equivalently, it is the minimum rank of a real $\cC \times \cX$ matrix whose sign pattern encodes $\cC$.  

Since for any $P \subseteq \mathbb{R}^d$, the corresponding class of homogeneous half-spaces $\cH_P$ is extremal and satisfies $\VCdim(\cH_P) \le d$,  \Cref{thm:intro_LCdim_extremal,thm:intro_SCD_LR_equivalence} imply the following bounds as a corollary. 

\begin{corollary}
\label{cor:signrank}
Every binary concept class $\cC$ over a finite domain satisfies 
\[\SCdim(\cC) \le \signrank(\cC)-1 \ \text{ and equivalently } \ \LR(\cC) \le \signrank(\cC).
\]
\end{corollary}

The inequality $\LR(\cC) \le \signrank(\cC)$ was conjectured by Chase et al.~\cite{chase2023replicabilitystabilitylearning}, who verified it in the case $\signrank(\cC)=2$. The full conjecture was later resolved in~\cite{blondal2025borsukulamreplicablelearninglargemargin} using an algorithmic argument based on averaging multiple runs of a large-margin classifier and rounding to an $\epsilon$-net. In contrast, \Cref{thm:intro_LCdim_extremal} provides a purely topological proof. Furthermore, rather than the geometry of points and half-spaces, it only relies on the fact that $\cH_P$ is extremal.

\subsection{Concluding remarks}
Arguably, the most intriguing open question regarding simplicial covering dimension (equivalently, list replicability~\cite[Section~3.4]{chase2023replicabilitystabilitylearning}) is whether it can be bounded by a function of the VC dimension.

\begin{question}
\label{q:VC_LC}
Is there a function $t:\mathbb{N} \to \mathbb{N}$ such that, for every concept class $\mathcal{C} \subseteq \set{\pm 1}^\cX$ over a finite domain $\cX$,  
\[
\SCdim(\cC) \le t(\VCdim(\cC))?
\]
\end{question} 

By \Cref{thm:intro_LCdim_extremal}, a positive answer to \Cref{q:extremal} would immediately imply a positive answer to \Cref{q:VC_LC}.

The dual class $\cC^*$ of $\cC$ is obtained by swapping the roles of hypotheses $c \in \cC$ and points $x \in \cX$, or equivalently, by transposing the matrix $(c(x))_{c \in \cC, x \in \cX}$. It follows from~\cite[Theorem D]{chase2023local} and \cref{thm:intro_SCD_LR_equivalence} that
\[
\SCdim(\cC) \geq \left\lfloor \frac{\VCdim(\cC^*)}{2} \right\rfloor.
\]
For the cube concept class $\mathcal{Q} = \set{\pm 1}^{\cX}$, we have
\[\VCdim(\mathcal{Q}) = |\cX| \text{ and }\VCdim(\mathcal{Q}^*) = \lfloor \log_2(|\cX|) \rfloor.\] 
Therefore, its dual class $\mathcal{Q}^*$ exhibits an exponentially large gap between the VC dimension and the simplicial covering dimension. Consequently, if the function $t(\cdot)$ of \Cref{q:VC_LC} exists, it must grow at least exponentially.

It is straightforward to verify that every class $\cC$ satisfies  $\VCdim(\cC^*) \le 2^{\VCdim(\cC)}$. Furthermore, if $\cE$ is an extremal class, then it was shown in~\cite{chase2024dual} that a stronger upper bound $\VCdim(\cE^*) \le 2\VCdim(\cE)+1$ holds. It would be interesting to investigate whether an analogous relationship exists between $\SCdim(\cC)$ and $\SCdim(\cC^*)$.

\begin{question}
What is the relation between $\SCdim(\cC)$ and $\SCdim(\cC^*)$, where $\cC^*$ denotes the dual class? For example, is it true that $\SCdim(\cC^*) \le 2^{\SCdim(\cC)}$?  
\end{question}

Finally, let us discuss the relationship between simplicial covering dimension and two other learning-theoretic complexity measures:  Littlestone dimension and spherical dimension.

\paragraph{Littlestone dimension.} Littlestone dimension is a refinement of VC dimension that determines the optimal number of mistakes in \emph{online learning}. In particular, VC dimension is a lower bound on Littlestone dimension. The celebrated result of~\cite[Theorem 23]{Alon_22_private_and_online}, together with the relation between global stability and list replicability, implies that every class with Littlestone dimension $d$ satisfies 
\[
\SCdim(\cC) \le 2^{2^{O(d)}}.
\]
On the other hand, $\SCdim(\cC)$ does not provide any upper bound on the Littlestone dimension: for $t>2$, the class $\mathcal{T}_t$ of threshold functions over the domain $[t-1]$ has Littlestone dimension $\lfloor \log_2(t) \rfloor$, yet it is an extremal case, and therefore, by \Cref{thm:intro_LCdim_extremal} satisfies $\SCdim(\cC)=1$.

\paragraph{Spherical dimension.} The spherical dimension $\SDdim(\cC)$ of a class $\cC$, introduced in~\cite{chornomaz2025spherical}, is the largest integer $d$ such that there exists a continuous antipodal map $f: \bS^d \to \Delta_\cC$.
In the notation of Matou\v{s}ek~\cite{MatousekBook}, this quantity is the \emph{$\bZ_2$-coindex} of $\Delta_\cC$. Spherical dimension provides a lower bound for simplicial covering dimension:
\[
\SCdim(\cC) \geq \left\lfloor \frac{\SDdim(\cC)}{2} \right\rfloor + 1.
\]
It is open whether the list replicability number and equivalently the simplicial covering dimension can be bounded above by a function of the spherical dimension.  A positive resolution of \Cref{q:VC_LC} would imply such a bound, since the spherical dimension admits tight lower bounds in terms of both VC and dual VC dimensions~\cite{chornomaz2025spherical}.

\subsection{Paper organization} 
For the reader’s convenience, in \Cref{sec:backg},  we briefly review basic concepts from topological combinatorics, including simplicial and cubical complexes, abstract complexes, and their subdivisions. Furthermore, we discuss the simplicial structure of $\Delta_\cC$ and the cubical complex $\Gamma_\cC$ defined by the strongly shattered sets. In \Cref{sec:basic_facts}, we prove a few elementary facts about the simplicial covering dimension that we will use in the proofs of our main theorems. \Cref{sec:LC_equivalent_to_LR} contains the proof of \Cref{thm:intro_SCD_LR_equivalence}, and \Cref{sec:intro_LCdim_extremal} contains the proof of our main result, \Cref{thm:intro_LCdim_extremal}.

\section{Background}
\label{sec:backg}

\paragraph{Notation.} For a positive integer $n \in \mathbb{N}$, we denote $[n] \coloneqq \set{1,\ldots,n}$.  Since $\cX$ is finite, all norms on $\mathbb{R}^\cX$ define the same topology. In this paper, we use the $\ell_1$ norm as it is more natural when working with distributions. Given $x \in X \subseteq \mathbb{R}^d$ and a radius $\epsilon>0$, we denote the $\ell_1$-open ball of radius $\epsilon$ around $x$ by 
\[ B_{\ell_1}(\epsilon,x) \coloneqq \set{y \in \mathbb{R}^d:~ \norm{x-y}_1 < \epsilon}.\] 
Given $A \subseteq \R^d$ and $\epsilon > 0$, define 
\[
    A^{(\epsilon)} \coloneqq \set{y \in \R^d:~ \norm{x-y}_1 < \epsilon \text{ for some } x\in A}.
\]

For $X,Y\subseteq \mathbb{R}^d$, we also define
\[d_{\ell_1}(X,Y)\coloneq \inf_{x\in X, y\in Y}\norm{x-y}_1 \]
to be the minimal $\ell_1$ distance between points of $X$ and $Y$.

We will denote the abstract simplicial and cubical complexes with capital English letters, such as $K, L, Q$. We use capital Greek letters such as $\Gamma$ and $\Delta$ to denote geometric simplicial and cubical complexes.

\subsection{Simplicial complexes}

A simplex $\sigma$ is the convex hull of a \emph{finite} affinely independent set $A$ in $\R^d$. The elements of $A$ are called the vertices of $\sigma$, and the dimension of $\sigma$ is defined by $\dim(\sigma) \coloneqq |A|-1$. Following the convention of~\cite{MatousekBook}, we assume that the empty set is a simplex of dimension $-1$.

The convex hull of any subset of the vertices of $\sigma$ is called a face of $\sigma$; note that every face is itself a simplex, and according to our convention, the empty set is a face of every simplex.   

\begin{definition}[Geometric Simplicial Complex]
\label{def:geomSimp}
A nonempty collection $\Delta$ of simplices in $\R^d$ is called a (geometric) simplicial complex if the following two conditions hold:
\begin{enumerate}
\item Every face of a simplex $\sigma \in \Delta$ also belongs to $\Delta$;

\item The intersection of any two simplices $\sigma_1,\sigma_2 \in \Delta$ is a face of both $\sigma_1$ and $\sigma_2$. 
\end{enumerate}
The dimension of $\Delta$ is the maximum of the dimensions of simplices in $\Delta$. 
\end{definition}

Every geometric simplicial complex $\Delta$ gives rise to a topological space on $\norm{\Delta}=\bigcup_{\sigma \in \Delta} \sigma$, called its \emph{polyhedron} or \emph{underlying space}. Giving each simplex its natural topology as a subspace of $\mathbb{R}^d$, one then topologizes $\norm{\Delta}$ by declaring a subset $C \subseteq \norm{\Delta}$ to be closed 
if and only if $C \cap \sigma$ is closed in $\sigma$, for each $\sigma \in \Delta$. In general, this topology does not coincide with the subspace topology inherited from $\R^d$; however, for \emph{finite} simplicial complexes, the two topologies agree, and we may simply regard $\norm{\Delta}$ as a subset of $\R^d$. 

The simplicial structure of a geometric simplicial complex is captured by a combinatorial object called an abstract simplicial complex. 

\begin{definition}[Abstract simplicial complex]
\label{def:AbsSimp}
An \emph{abstract simplicial complex} on a set $V$  is a collection $K$ of finite subsets of $V$ such that  $ s \in K$ and $ t \subseteq s$ imply $ t \in K$. 
The elements of $K$ are called \emph{simplices} of $K$, and the elements of $V$ are called the vertices of $K$. The dimension of $K$ is defined by $\dim(K) \coloneqq \max_{s \in K} |s|-1$. 
\end{definition} 

Every geometric simplicial complex $\Delta$ determines an abstract simplicial complex $K$ with vertex set $V = V(\Delta)$, where each $s\in K$  is the set of vertices of a simplex in $\Delta$. In this case, we say that $\Delta$ is a \emph{geometric realization} of $K$, and the corresponding topological space $\norm{\Delta}$ is called a \emph{polyhedron} of $K$.

Conversely, every \emph{finite} abstract simplicial complex $K$ admits a geometric realization in some Euclidean space $\mathbb{R}^d$. For instance, one may map the vertices of $K$ to an affinely independent set of $|V|$ points in $\mathbb{R}^{|V|-1}$ and realize each simplex $s \in K$ as the convex hull of the corresponding points. This yields a geometric realization of $K$ in $\mathbb{R}^{|V|-1}$.

The following simple proposition implies that all polyhedra of a finite abstract simplicial complex are homeomorphic, so the associated topological space is unique up to homeomorphism.

\begin{proposition}[{\cite[Proposition 1.5.4]{MatousekBook}}] 
\label{prop:abstract_Simp}
Let $\Delta$ and $\Delta'$ be geometric simplicial complexes, and let $K$ and $K'$ be their associated abstract simplicial complexes. Suppose $f \colon V(K) \to V(K')$ is a simplicial map; that is, it maps every face of $K$ to a face of $K'$. Then there exists a continuous map $\rho \colon \norm{\Delta} \to \norm{\Delta'}$ such that:
\begin{itemize}
    \item if $f$ is injective, then $\rho$ is injective;
    \item if $f$ is an isomorphism, then $\rho$ is a homeomorphism.
\end{itemize}
\end{proposition}

Finally, we discuss the barycentric subdivision of a simplicial complex. Suppose $K$ is an abstract simplicial complex and $\Delta$ is its geometric realization. Let $\sigma \in \Delta$ be a simplex and let $x \in \sigma$ be a point. Then the \emph{barycentric coordinates} of $x$ in $\sigma$ are the unique solution to the equation $\sum_{v\in V(\sigma)}\alpha_v v=x$ and $\sum_{v\in V(\sigma)}\alpha_v=1$. It is not hard to see there is a smallest simplex $\sigma$, called the \emph{support of $x$} and denoted $\supp_\Delta(x)$, that contains $x$. In that case, all barycentric coordinates satisfy $\alpha_v>0$.

\begin{definition}[Abstract subdivision]\label{def:subd_abst_simplex} The barycentric subdivision of an abstract simplicial complex $K$, denoted $K_1$, is defined as the abstract simplicial complex whose vertices are nonempty simplices of $K$, i.e. $V(K_1)=K\setminus \set{\emptyset}$, and whose simplices are chains of non-empty simplices of $K$ ordered by inclusion.  
\end{definition}

The geometric realization of $K_1$ is natural.

\begin{definition}[Geometric subdivision] Let $\Delta$ be the geometric realization of $K$. For each nonempty simplex $\sigma \in \Delta$, define $v_{\sigma}$ to be its \emph{barycenter}, that is $v_{\sigma}=\tfrac{1}{|\sigma|}\sum_{v \in \sigma} v $. Having defined the coordinates of each vertex in the subdivision, we extend linearly for points in each simplex and obtain a geometric realization of $K_1$, denoted $\Delta_1$.
\end{definition}
Note that $\norm{\Delta_1}$ and $\norm{\Delta}$ are homeomorphic.

\subsection{Cubical complexes}
Cubical complexes serve as a combinatorial and geometric framework analogous to simplicial complexes, but built from cubes instead of simplices.  

\begin{definition}[Elementary cubes]
An \emph{elementary interval} in $\mathbb{R}$ is a closed interval of the form $[a, a]$ or $[a, a+1]$ for some $a \in \mathbb{Z}$. An \emph{elementary cube} in $\mathbb{R}^d$ is a product $\kappa = I_1 \times \cdots \times I_d$, where each $I_i$ is an elementary interval. The \emph{dimension} of $\kappa$ is the number of factors $I_i$ of the form $[a, a+1]$.
\end{definition}

\begin{definition}[Geometric cubical complex]
A \emph{geometric cubical complex} $\Gamma$ in $\mathbb{R}^d$ is a \emph{finite} non-empty collection of elementary cubes such that:
\begin{enumerate}
    \item If $\kappa \in \Gamma$ and $\tau$ is a face of $\kappa$ (obtained by replacing one or more factors $[a,a+1]$ with $[a,a]$ or $[a+1,a+1]$), then $\tau \in \Gamma$;
    \item If $\kappa_1, \kappa_2 \in \Gamma$, then $\kappa_1 \cap \kappa_2$ is either empty or is a face of both $\kappa_1$ and $\kappa_2$.
\end{enumerate} The \emph{polyhedron} $\norm{\Gamma}$ is the union of all cubes in $\Gamma$, equipped with the subspace topology from $\mathbb{R}^d$.
\end{definition}

Note that sometimes we allow the elementary geometric cubes in a geometric cubical complex to be scaled by some factor. Nevertheless, this does not change any of the topological properties of the cubical complex.

We proceed with the corresponding theory of abstract cubical complexes. The \emph{standard abstract $n$-cube} is represented by the set $\{0,1\}^n$. A \emph{face} of the standard abstract $n$-cube is a product $q_1 \times \cdots \times q_n$, where each $q_i$ is a non-empty subset of $\{0,1\}$.

\begin{definition}[Abstract cubical complex \cite{farley2003finiteness}]\label{def:abst_cub_com}
 An abstract cubical complex on a non-empty vertex set $V$ is a collection $Q$ of non-empty subsets of $V$, called cubes, satisfying:
 \begin{enumerate}
     \item 
$Q$ is a cover of $V$.
\item For any $q_1, q_2 \in Q, q_1 \cap q_2 \in Q$ or $q_1 \cap q_2=\emptyset$.
\item For any $q \in Q$ there is a bijection $\phi_q: q \rightarrow\{0,1\}^n$, for some $n$, satisfying: if $q_1 \subseteq q$, then $q_1 \in Q$ if and only if $\phi_q\left(q_1\right)$ is a face of $\{0,1\}^n$.
\end{enumerate}
The elements of $V$ are called the vertices or $0$-cubes of  $\cQ$. In general, we say $q\in \cQ$ is an \emph{$n$-cube} if $|q|=2^n$.
\end{definition}

As for simplicial complexes, each geometric cubical complex $\Gamma$ has a corresponding abstract cubical complex $Q$: each abstract cube represents the set of vertices of its respective geometric cube. We again say $\Gamma$ is a geometric realization of $Q$. To better understand the structure of an abstract cubical complex, it is often useful to study its barycentric subdivision, which is a \emph{simplicial complex} defined as follows. 

\begin{definition}[Subdivision of a cubical complex \cite{farley2003finiteness}]\label{def:subd_cubical}
The barycentric subdivision of an abstract cubical complex $Q$, denoted $Q_1$, is the abstract \emph{simplicial complex} whose vertices are the cubes in $Q$ and whose simplices are chains of cubes ordered by inclusion.
\end{definition}

\subsection{Simplicial complex of a class}
To describe the simplicial structure of the set of realizable distributions, we need a few notations.

A \emph{partial concept} on a domain $\cX$ is a labeling $h \colon \cX \to \set{\pm 1, *}$, where $h(x) = *$ means that $h$ is \emph{undefined} at $x$. The \emph{support} of $h$ is $\supp(h) \coloneqq \set{x \in \cX : h(x) \neq *}$. Denote by $h_*$ the partial concept with empty support.

If $c\in \set{\pm1}^{\cX}$ and $h \in\set{\pm1,*}^{\cX}$ is a partial concept such that $c$ and $h$ agree on $\supp(h)$, then $c$ is called a \emph{completion} of $h$. More generally, we say that a partial concept $h_1$ \emph{extends} $h_2$ if $\supp(h_2) \subseteq \supp(h_1)$ and $h_1$ and $h_2$ agree on $\supp(h_2)$. We denote this by $h_1\geq h_2$.

Given a concept class $\cC \subseteq \set{\pm 1}^\cX$, we say that a partial concept $h$ is \emph{realizable by $\cC$} if there exists a concept $c \in \cC$ that is a completion of $h$.

\begin{definition}[Abstract simplicial complex of a class]
Given a concept class $\cC$, let $\cC^\bullet \subset \set{\pm 1,*}^\cX$ be the set of all realizable partial concepts of $\cC$. For each  $h\in \cC^\bullet$, define the set
\[
s_h\coloneq\set{(x,h(x)):~x\in \supp(h)}.
\]
The collection $\set{s_h}_{h\in \cC^\bullet}$ forms an abstract simplicial complex, which we denote by $D_\cC$. 
\end{definition}

In particular, note that $V(D_\cC)\subseteq \cX\times\set{\pm 1}$ and that the all-$*$ labeling $h_*$ corresponds to the empty set. In addition, the maximal faces of $D_\cC$ are exactly those defined by concepts in $\cC$, i.e., $s_c$ where $c\in \cC$. 
 
Since there is a bijective correspondence between simplices of $D_\cC$ and partial concepts realizable by $\cC$, we will often treat $h \in \cC^\bullet$ directly as simplex of $D_\cC$.

Notice $D_\cC$ has a natural geometric realization in $\mathbb{R}^{\cX}$. Indeed, let $\{\mathbf{e}_x\}_{x \in \cX}$ be the \emph{standard basis} of $\mathbb{R}^\cX$.
 
\begin{definition}[Geometric simplicial complex of a class]
 For each realizable partial concept $h \in \cC^\bullet$, define the geometric simplex
 \[
 \sigma_h \coloneqq \conv \bigl(\{ \sign(h(x))  \mathbf{e}_x:~x \in \supp(h) \}\bigr),
 \]
 which corresponds to the abstract simplex $s_h$ in $D_{\cC}$. The collection $\{ \sigma_h : h \in \cC^\bullet \}$ forms a geometric simplicial complex in $\mathbb{R}^{\cX}$, denoted $\Delta_\cC$. 
\end{definition}
For convenience, we will also use $\Delta_{\cC}$ (instead of $\norm{\Delta_{\cC}}$) to denote $
\bigcup_{h \in \cC^\bullet} \sigma_h$, i.e., the underlying polyhedron of $\Delta_\cC$. The meaning will be clear from context. Note that this is consistent with the definition of $\Delta_\cC$ in \eqref{eq:definition_D_C}.

Similarly to the abstract setting, the maximal simplices $\sigma_h$ are exactly those with $h \in \cC$. Also note that for $h\in\cC$, we have $\sigma_h$ equals $B_h$, from \eqref{eq:zero-loss_set},  as subsets of $\mathbb{R}^{\cX}$. Nevertheless, we prefer to keep their notation separate in order to differentiate between the two.

Next, we consider the subdivision of $D_\cC$, denoted $D_{\cC,1}$. Recall that in a \emph{simplicial subdivision}, simplices correspond to chains of \emph{non-empty} simplices in the original complex (\cref{def:subd_abst_simplex}). Here, we call a sequence of realizable partial concepts $h_1, \dots, h_k$ a \emph{chain} if 
\[ h_1\geq \ldots \geq h_k,\]  
i.e., each partial concept extends the next one in the sequence. The simplices of $D_{\cC,1}$ are in bijective correspondence with chains of realizable partial concepts in $\cC^\bullet\setminus \set{h_{*}}$ where $h_{*}$ is excluded as it corresponds to the empty simplex. In particular, $V(D_{\cC,1}) = \cC^\bullet\setminus \set{h_{*}}$.  

\input{fig_Delta}

Since $D_{\cC,1}$ is a subdivision of $D_\cC$, it inherits a natural geometric realization. To a vertex (partial concept) $h \in \Delta_{\cC,1}$ we associate the uniform distribution over $\supp(h)$. This extends linearly to all points in $\Delta_{\cC,1}$. For instance, if $\mu$ is a point in the simplex $h_1 \geq \ldots \geq h_k$, then $\mu=\sum_{i=1}^{k}\alpha_{i}h_i$ where the barycentric coordinates satisfy $\sum_{i=1}^{k}\alpha_i=1$ and each partial concept $h_i$ represents the coordinates given by the uniform distribution over $\supp(h_i)$. We denote by $\supp_{\Delta_{\cC,1}}(\mu)$ the minimal simplex $h_1\geq\dots\geq h_k$ containing $\mu$, i.e., the simplex for which all barycentric coordinates $\alpha_i>0$.  

Finally, note $\Delta_{\cC}$ and $\Delta_{\cC,1}$ are equal as subsets of $\mathbb{R}^{\cX}$ and are, in particular, homeomorphic.
\subsection{The cubical complex of a class}

\begin{definition}[Abstract cubical complex of a class]
\label{def:Abs_cub_class}Let $\cC$ be a concept class. The abstract cubical complex $Q_{\cC}$ of $C$ consists of all subsets $q = \set{c_1, \ldots, c_k}$ of $\cC$ such that $q$ strongly shatters some subset $S \subseteq \cX$.   
\end{definition}

It is not difficult to check that $\cQ_{\cC}$ satisfies the properties in the definition of an abstract cubical complex (\cref{def:abst_cub_com}). Note that the set of vertices (0-cubes) is $V(Q_{\cC})=\cC$. In general, if $q = \set{c_1, \cdots, c_k}$ strongly shatters a set $S \subseteq \cX$, then $q$ is an $|S|$-cube. Finally, observe $q$ can be represented by the partial concept $h_q$ defined by
    \begin{align*}
        h_q(x) = \begin{cases}
            * &\text{if $x \in S$}\\
            c_1(x)=\cdots=c_k(x) &\text{otherwise}
        \end{cases}
    \end{align*}
as the set $\set{c_1, \dots, c_k}$ consists of all possible completions of $h_q$ to a concept in $\set{\pm1}^{\cX}$. In fact, there is a bijective correspondence between cubes of $\cQ_{\cC}$ and partial concepts $h\in\set{\pm1,*}^{\cX}$ whose every completion belongs to $\cC$. For this reason, we sometimes treat partial concepts directly as cubes.

Next, we construct a natural geometric realization of $Q_{\cC}$.

\begin{definition}[Geometric cubical complex of a class]
    For each abstract cube $q=\set{c_1,\dots,c_k} \in \cQ_{\cC}$, define an elementary geometric cube $\kappa_q$ by
\[
\kappa_q \coloneq \set{y\in [-1,1]^{\cX}|~y(x)=h_q(x)~\text{for} ~x\in\supp(h_q)}.
\]
The collection $\set{\kappa_q }_{q\in \cQ_{\cC}}$ is a geometric cubical complex denoted by $\Gamma_{\cC}$.   
\end{definition}
Note each concept ($0$-cube) $c\in \cC$ corresponds to a point $y_c \in [-1,1]^{\cX}$ on the $\ell_\infty$ sphere, defined by 
\[
        y_c = (\sign(c(x)) : x \in X).
\]
In general, observe that the dimension of each cube $\kappa_q \in \Gamma_{\cC}$ equals the cardinality of the subset $S\subseteq\cX$ that is strongly shattered by $q$.

For convenience, we will use $\Gamma_{\cC}$ (instead of $\norm{\Gamma_{\cC}}$) to denote $
\bigcup_{q \in \cQ_{\cC}} \kappa_q$, i.e., the underlying polyhedron of $\Gamma_\cC$. The meaning will be clear from context.

Next, we consider the barycentric subdivision of the abstract cubical complex 
of a concept class $\cC$, denoted $Q_{\cC,1}$. By \cref{def:subd_cubical}, $Q_{\cC,1}$ is a simplicial complex whose vertices are the cubes of $Q_{\cC}$ and whose simplices are chains of ascending cubes. Hence, there is a bijective correspondence between simplices of  $Q_{\cC,1}$ and chains of partial concepts $h_1\geq\dots\geq h_k$ such that for any $i$, all completions of $h_i$ are in $\cC$.

\input{fig_3dcube}

Now $Q_{\cC,1}$ also has a natural geometric realization. Indeed, map each vertex (partial concept) $h$ to
\[y_h = \begin{cases}
    \sign(h(x)) & \text{if~} x \in \supp(h)\\ 
    0 &  \text{otherwise}
\end{cases}\]
and extend linearly for each simplex $h_1\geq\dots\geq h_k$ in $Q_{\cC,1}$. We denote this geometric realization of $Q_{\cC,1}$ by $\Gamma_{\cC,1}$. In addition, note $\Gamma_{\cC,1}$ is consistent with the geometric realization $\Gamma_{\cC}$ in the sense that they are equal as subsets of $\mathbb{R}^{\cX}$ and are, in particular, homeomorphic. Hence, one may study the geometry of $\Gamma_C$ from both a simplicial and cubical perspective.

\begin{lemma}[\cite{chornomaz2025spherical}] \label{lemma:full_subcomplex}Suppose $\cC\neq \set{\pm1}^{\cX}$ is a concept class that is not the binary cube. Then $Q_{\cC,1}$ is a full subcomplex of $D_{\cC,1}$, that is, if $s$ is a simplex in $D_{\cC,1}$ and for all vertices $v\in s$, we have $v\in Q_{\cC,1}$, then $s\in Q_{\cC,1}$.
    \end{lemma}
\begin{proof}
Since $\cC\neq \set{\pm1}^{\cX}$, the all-$*$ labelling $h_*$ is not a cube. Hence, the vertex set $V(Q_{\cC,1})\subseteq \cC^\bullet\setminus\set{h_*} = V(D_{\cC,1})$. Each chain of cubes $h_1\geq \cdots \geq h_k$ (simplex of $Q_{\cC,1}$) is also a chain of realizable partial concepts (simplex of $D_{\cC,1}$). Thus, $Q_{\cC,1}$ is a subcomplex of $D_{\cC,1}$. Finally, if $h_1\geq\dots \geq h_k$ is a chain of realizable partial concepts such that each $h_i$ is a cube, then we immediately get a chain of cubes. We conclude  $Q_{\cC,1}$ is a full subcomplex.
\end{proof}

The reason $Q_{\cC,1}$ does not embed into $D_{\cC,1}$ for $\cC=\set{\pm 1}^{\cX}$ is that $h_*$ is not a vertex of $D_{\cC,1}$ by definition, while it is a vertex of $Q_{\cC,1}$ because all completions of $h_*$ are in $\set{\pm 1}^{\cX}$. Thus, so long as $\cC$ is not the binary cube, there is a simplicial embedding $\Gamma_{\cC,1}\hookrightarrow \Delta_{\cC,1}$ which can also be thought of as an embedding $\Gamma_{\cC} \hookrightarrow \Delta_{\cC}$.
This embedding can be realized as $\ell_1$ normalization:

\begin{equation}\label{eq:embedding}
\begin{aligned}
    f_{\cC} &\colon \Gamma_{\cC} \hookrightarrow \Delta_{\cC}, \\
    f_{\cC}&\colon y \mapsto   y / \norm{y}_1.
\end{aligned}
\end{equation}
For $\cC \neq \set{\pm 1}^{\cX}$, we shall denote 
\begin{equation}\label{eq:embedding_image}
\begin{aligned}
    \widetilde{\Gamma}_{\cC} &\coloneqq f_{\cC} (\Gamma_{\cC}) \subseteq \Delta_{\cC}, \\
    \widetilde{\Gamma}_{\cC,1} &\coloneqq f_{\cC} (\Gamma_{\cC,1}) \subseteq \Delta_{\cC,1}.
\end{aligned}
\end{equation}

\input{fig_embedding}

\section{Dimension theory toolkit}
\label{sec:basic_facts}
In this section, we prove a few elementary dimension theory facts, which we will invoke repeatedly in the proofs of our main theorems (\Cref{thm:intro_LCdim_extremal,thm:intro_SCD_LR_equivalence}).

Recall that given $A \subseteq \mathbb{R}^n$ and $\epsilon > 0$, we defined  
\[A^{(\epsilon)} \coloneqq \set{y \in \mathbb{R}^n:~ \norm{x-y}_1 < \epsilon \text{ for some } x\in A}.\]

As discussed in \Cref{rem:outside_X}, when working with topological spaces of the form $X \subseteq \mathbb{R}^n$, we write $\set{A_i}_{i \in I} \prec \set{B_j}_{j \in J}$ to mean $\set{A_i \cap X}_{i \in I} \prec \set{B_j \cap X}_{j \in J}$.

The following lemma gives a useful metric tool for constructing refinements.

\begin{lemma}\label{lemma:metric_refinement}
    Let $\set{A_i}_{i \in I} \prec \set{B_j}_{j \in J}$ be covers of $X \subseteq \mathbb{R}^n$ such that  
    \begin{itemize}
    \item $A_i \cap X$ are open for all $i \in I$;
    \item $J$ is finite and $B_j \cap X$ are compact in $X$ for all $j \in J$.
    \end{itemize}
    There is some $\alpha > 0$ such that $\set{A_i}_{i \in I} \prec \Set{B_j^{(\alpha)}}_{j \in J}$.
\end{lemma}
\begin{proof}
    Since $\set{B_j}_{j \in J}$ is a refinement of the open cover $\set{A_i}_{i \in I}$, each $B_j$ is contained in some open $A_i$. If $B_j$ is empty, then $B_j^{(\alpha)}$ is also empty and is therefore contained in $A_i$ for any $\alpha > 0$. Likewise, if $A_i \supseteq X$, then $B_j^{(\alpha)} \cap X \subseteq A_i$ for any $\alpha > 0$.
    
    Otherwise, $X \setminus A_i$ is a nonempty closed set, $B_j \cap X$ is a nonempty compact set, and $B_j \cap (X \setminus A_i) = \emptyset$. Together these imply that the distance $\beta_j \coloneqq d_{\ell_1}(X \setminus A_i,B_j \cap X)$ is positive, and so $B_j^{(\beta_j)} \cap X \subseteq A_i$. Since there are finitely many $j\in J$, taking $\alpha \coloneqq \min_{j \in J}\beta_j$ completes the proof.
\end{proof}

The next lemma allows us to pass from an open cover to a closed refinement without increasing order.

\begin{lemma}\label{lemma:open_cover_to_closed}
    Any open cover $\set{U_i}_{i \in I}$ of a compact set $X\subseteq\R^n$ has a closed refinement $\set{F_i}_{i\in I}$, and in particular, $\ord(\set{F_i}_{i \in I}) \leq \ord(\set{U_i}_{i \in I})$.
\end{lemma}
\begin{proof}
    Since $\set{U_i}_{i \in I}$ is a cover of $X$, for each $x \in X$ there exists an $i\in I$ such that $x \in U_i$. Because $U_i$ is open and $X$ is a subset of $\R^n$, there exists some open $V_x$ such that
    \[
        x \in V_x \subseteq \overline{V_x} \subseteq U_i.
    \]
    By the compactness of $X$, the open cover $\set{V_x}_{x \in X}$ admits a finite subcover $\set{V_{x_j}}_{j \in J}$. Hence, $\set{\overline{V_{x_j}}}_{j \in J}$ is a finite closed cover. Using this cover, we define for each $i \in I$ a closed set
    \[
        F_i \coloneq \bigcup \Set{\overline{V_{x_j}} : j \in J, \overline{V_{x_j}} \subseteq U_i}.
    \]
    By construction, each $\overline{V_{x_j}}$ is contained in some $F_i$, and so $\set{F_i}_{i \in I}$ is a closed cover of $X$. Furthermore, for each distinct $F_i$ and point $x \in X$ such that $F_i$ contains $x$, there is a distinct $U_i$ such that $x \in F_i \subseteq U_i$. We conclude that $\set{F_i}_{i \in I}$ refines $\set{U_i}_{i \in I}$ without increasing order, as claimed.
\end{proof}

Here we present the usual formulation of the Lebesgue covering theorem with closed sets, which can be used to derive the open version stated in \cref{thm:lebesgue}.

\begin{theorem}[{Lebesgue covering theorem~\cite[Theorem IV 2.]{dimensiontheory1941}}]\label{thm:lebesgue_covering_thm}
    Suppose a $d$-dimensional cube $[-1,1]^d \subseteq \R^d$ is covered by a finite family $\cA$ of closed sets, none of which contains points of opposite faces of the cube.
    Then $\ord(\cA) \geq d$.
\end{theorem} 
\begin{proof}[Proof of \cref{thm:lebesgue}]
    Suppose that $\cA$ is a finite family of open sets covering the cube $[-1,1]^d$ such that no member of $\cA$ contains points of opposite faces of the cube.
    By \cref{lemma:open_cover_to_closed}, there exists a closed refinement $\cF$ of $\cA$ that is also a cover, such that $\ord(\cF) \leq \ord(\cA)$.
    
    Since $\cF$ refines $\cA$, it also has no member set containing points of opposite faces of the cube. Applying \cref{thm:lebesgue_covering_thm} to the closed cover $\cF$ yields that $\ord(\cF) \ge d$. We conclude that $\ord(\cA) \ge d$ as claimed.
\end{proof}

Finally, we record the following useful lemma, which shows that for any closed cover of order $d$ of a compact space, there exists a sufficiently small radius $\alpha>0$ such that every $\alpha$-neighbourhood around any point meets at most $d+1$ sets of the cover. 

\begin{lemma}\label{lemma:closed_cover_rounding}
    If a finite closed cover $\cF$ of a compact set $X\subseteq\R^n$ has order $d$, then there is some $\alpha > 0$ such that, for any $x \in X$, the ball $ B_{\ell_1}(\alpha,x) \cap X$ intersects at most $d+1$ sets of $\cF$.
\end{lemma}
\begin{proof}
    For the sake of contradiction, assume that for every $\alpha > 0$ there is some $x \in X$ such that $B_{\ell_1}(\alpha,x) \cap X$ intersects at least $d+2$ sets of $\cF$. Then we may pick some sequence $\set{x_k}_{k=1}^\infty$ such that $B_{\ell_1}(\frac{1}{2^k},x_k) \cap X$ intersects at least $d+2$ sets of $\cF$ for every $k\in\mathbb{N}$.
    
    Since $X$ is compact, $\set{x_k}_{k=1}^\infty$ has a subsequence $\set{x_{k_i}}_{i=1}^\infty$ which converges to some $x^* \in X$. Furthermore $\cF$ is finite, so there are sets $F_1,\dots,F_{d+2} \in \cF$, each of which intersects $B_{\ell_1}(\frac{1}{2^{k_i}},x_{k_i}) \cap X$   for infinitely many $i \in \mathbb{N}$. Putting these observations together, we see that for any $j \in [d+2]$, we have
    \[
        d_{\ell_1}(F_j,x^*) \leq d_{\ell_1}(F_j,x_{k_i}) + d_{\ell_1}(x_{k_i},x^*),
    \]
    where the sum of the right-hand terms can be made arbitrarily small for an appropriate choice of $i \in \mathbb{N}$. Thus, each of the closed sets $F_1,\dots,F_{d+2}$ has $x^*$ as a limit point, whereby $x^*$ is contained in $d+2$ sets of $\cF$, which contradicts that $\cF$ has order $d$.
\end{proof}

\section{\Cref{{thm:intro_SCD_LR_equivalence}}: Simplicial covering dimension and replicability}\label{sec:LC_equivalent_to_LR}

In the introduction, when discussing the problem of learning a class $\cC$, we assumed that the learner is allowed to output any hypothesis $h \in \set{\pm 1}^\cX$ with small loss. However, in some settings, one restricts the learner to output hypotheses from some class $\cH \supseteq \cC$. For example, in \emph{proper learning}, the output of the learner must be from the original class $\cC$. 

\paragraph{A more general setting.} Let us denote by $\LR_\cH(\cC)$ the list replicability number of $\cC$, with the extra requirement that the list in \Cref{def:list} must consist of hypotheses in $\cH$. With this notation, the unrestricted case is $\LR(\cC) = \LR_{\set{\pm 1}^\cX}(\cC)$.
Similarly, define \[\SCdim_{\cH}(\cC)\coloneq \SCdim_{\Delta_\cH}(\Delta_\cC),\] so that $\SCdim(\cC) = \SCdim_{\set{\pm 1}^\cX}(\cC)$ in the unrestricted setting.

The following theorem, a generalization of \Cref{thm:intro_SCD_LR_equivalence}, establishes a precise quantitative connection between simplicial covering dimension and list replicability.

\begin{theorem}[General form of \Cref{thm:intro_SCD_LR_equivalence}]\label{thm:SCD_LR_equivalence}
    For any finite domain $\cX$, concept class $\cC \subseteq \set{\pm 1}^\cX$, and hypothesis class $\cH\subseteq \set{\pm 1}^\cX $ such that $\cC \subseteq \cH$, we have $\LR_{\cH}(\cC) = \SCdim_{\cH}(\cC)+1$. 
\end{theorem}

\begin{proof}
    For convenience, let $d \coloneqq \SCdim_{\cH}(\cC)$ and let $L \coloneqq \LR_\cH(\cC)$. Recall that by an abuse of notation, we use $\Delta_\cC$ to denote both the simplicial complex and its polyhedron $\norm{\Delta_\cC}$. Since we are working with the topological space $\Delta_\cC \subseteq \mathbb{R}^\cX$, recall from \Cref{rem:outside_X} that we write $\cA \prec \cB$ to mean $\set{A_i \cap \Delta_\cC}_{i \in I} \prec \set{B_j \cap \Delta_\cC}_{j \in J}$.
    
    First, we will show that $d \leq L - 1$. This amounts to proving that if $\cU$ is an open cover of $\Delta_\cC$, such that
    $\cU \prec \set{\sigma}_{\sigma \in \Delta_\cH}$,
    then $\cU$ has a refinement of order at most $L - 1$ that is an open cover of $\Delta_\cC$.
    We begin by applying \cref{lemma:metric_refinement} to produce an $\alpha>0$ such that
    $\cU \prec \set{\sigma^{(\alpha)}}_{\sigma \in \Delta_\cH}$.
    Fix $0<\epsilon<\frac{\alpha}{4}$ and $0<\delta<\frac{1}{2(L+1)}$.
    Since the list replicability number of $\cC$ with respect to $\cH$ is $L$, there exists an $(\epsilon, L)$-list replicable learner $\bm\cA$ for $\cC$ with outputs in $\cH$ and with sample complexity $n \coloneqq n(\epsilon,\delta)$.
    Using this learning rule, for each $h \in \cH$, define an open set
    \[
        V_h \coloneqq \Set{\mu \in \Delta_{\cC} :~ \Prob_{S\sim \mu^n}[\bm\cA(S) = h] > 
        \frac{1-2\delta}{L} \text{ and } \loss_\mu(h) < 2\epsilon}.
    \]
    We will show that the family $\set{V_h}_{h \in \cH}$ is an open cover of $\Delta_\cC$, and is also indeed a refinement of $\cU$ with order at most $L-1$.

    The sets $V_h$ are open in $\Delta_\cC$ by continuity of $\Prob_{S\sim \mu^n}[\bm\cA(S) = h]$ and $\loss_\mu(h)$ in $\mu$.
    To see why $\set{V_h}_{h \in \cH}$ covers $\Delta_\cC$, first recall that every $\mu \in\Delta_\cC$ is a distribution realizable by $\cC$, and so the list replicable learner $\bm\cA$ guarantees a list $\set{h_1,\dots,h_L} \subseteq \cH$ such that
    \[
        \Prob_{S\sim \mu^n}[\bm\cA(S) \in \set{h_1,\dots,h_L}] \geq 1-\delta
    \]
    and $\loss_\mu (h_i) \leq \epsilon$ for each $i \in [L]$. It follows that there is some $h \in \set{h_1,\dots,h_L}$ such that
    \[
        \Prob_{S\sim \mu^n}[\bm\cA(S) = h] \geq \frac{1-\delta}{L} > \frac{1-2\delta}{L}
        \quad \text{and} \quad
        \loss_\mu(h) \leq \epsilon < 2\epsilon,
    \]
    whereby $\mu \in V_h$. Hence, $\set{V_h}_{h \in \cH}$ covers $\Delta_\cC$.
    
    As for the order of $\set{V_h}_{h \in \cH}$, suppose for contradiction that $h_1,\dots,h_{L+1} \in \cH$ are distinct hypotheses such that
    \[
        \mu \in \bigcap_{i \in [L+1]} V_{h_{i}}
    \]
    for some $\mu \in \Delta_\cC$. It follows that the probability of outputting each of the distinct hypotheses $h_1,\dots,h_{L+1}$ is greater than $\frac{1-2\delta}{L}$. As these events are disjoint, we deduce that
    \[
        \Prob_{S\sim \mu^n}[\bm\cA(S) \in \set{h_1,\dots,h_{L+1}}] > 
        (L+1)\cdot\frac{1-2\delta}{L} >
        \frac{L+1}{L} \cdot \left(1-\frac{1}{(L+1)}\right) = 1,
    \]
    which gives the desired contradiction.
    
     It remains to argue that
    \[
        \cU \prec \set{\sigma^{(\alpha)}}_{\sigma \in \Delta_\cH} \prec \set{V_h}_{h \in \cH}.
    \]
    The first refinement is already verified by the choice of $\alpha$, and the second refinement holds by construction since $\loss_\mu(h) < 2\epsilon < \frac{\alpha}{2}$ and  
    \[\loss_\mu(h) = \sum_{\substack{x \in \supp(\mu) \\ h(x) \neq \sign(\mu(x))} } |\mu(x)| =\frac{1}{2}\norm{\mu-|\mu| h}_1\ge \frac{1}{2} d_{\ell_1}(\mu,\sigma_h), \]  
     imply $\mu \in \sigma_h^{(\alpha)}$.

    The second half of the proof is to show that $L \leq d+1$. To this end, fix $\epsilon > 0$. If the relative simplicial covering dimension of $\cC$ with respect to $\cH$ is $d$, then there exists a refinement $\cU$ of
    $\set{B_h^{(\epsilon/2)}}_{h\in\cH}$
    such that $\cU$ has order no more than $d$ and is an open cover of $\Delta_\cC$. Now pick any ordering on $\cH$ and for $h\in \cH$ let
    \[
        V_h \coloneqq \bigcup \Set{U \in \cU :~ \text{$h$ is the first hypothesis in $\cH$ satisfying $U \subseteq B_h^{(\epsilon/2)}$}}.
    \]
    Note that by definition, $\mcV \coloneqq \set{V_h}_{h\in\cH}$ is a refinement of $\set{B_h^{(\epsilon/2)}}_{h\in\cH}$ and has order no more than $d$.
    By \cref{lemma:open_cover_to_closed}, this open cover $\mcV$ has a closed refinement $\cF$ (also indexed by $\cH$) of order at most $d$. Lastly, applying \cref{lemma:closed_cover_rounding} guarantees a $\beta > 0$ such that $\set{\nu \in \Delta_\cC :~ \norm{\nu-\mu}_1 \le \beta}$ intersects at most $d+1$ sets of $\cF$ for any $\mu \in \Delta_\cC$.

    Now we are ready to give an $(\epsilon,d+1$)-list replicable learner for $\cC$. For any distribution $\mu \in \Delta_\cC$ and sample $S \sim \mu^k$ for some $k \in \mathbb{N}$, let $\hat\mu$ denote the empirical estimate of $\mu$ using $S$. Since the distributions in $\Delta_\cC$ are defined on a finite domain, there exists, for any $\delta>0$, some positive integer $n \coloneqq n(\epsilon,\beta,\delta)$ such that
    \[
        \Prob_{S\sim \mu^n}\left[
            \norm{\mu-\hat\mu}_1 < \min\left(
                \frac{\epsilon}{2},\beta
            \right)
        \right] \geq 1-\delta
    \]
    for all $\mu \in \Delta_\cC$.

    Given a replicability parameter $\delta > 0$, an unknown distribution $\mu$, and a sample $S \sim \mu^n$, the learning rule is as follows:
    \begin{enumerate}
        \item Let $\hat\mu$ be the empirical estimate of $\mu$ using $S$.
        \item Select an arbitrary closed set $F_h \in \cF$ containing $\hat\mu$.
        \item Output $h$.
    \end{enumerate}
    First, let us check accuracy. We have that
    \[
        \loss_\mu(h)\leq \loss_{\hat\mu}(h) + \norm{\mu-\hat\mu}_1,
    \]
    where $\loss_{\hat\mu}(h) \leq \frac{\epsilon}{2}$ because $\hat\mu \in F_h \subseteq B_h^{(\epsilon/2)}$, and $\norm{\mu-\hat\mu}_1 \leq \epsilon/2$ with probability at least $1-\delta$.
    Second, we check list-replicability. With probability at least $1-\delta$ we have $\norm{\mu-\hat\mu}_1 < \beta$, in which case there are by choice of $\beta$ at most $d+1$ closed sets of $\cF$ that could contain $\hat\mu$.
\end{proof}

\section{\Cref{thm:intro_LCdim_extremal}: Simplicial covering dimension of extremal classes}
\label{sec:intro_LCdim_extremal}

In this section, we classify the simplicial covering dimension $\SCdim(\cE)$ of extremal classes. The argument is split between a lower bound given in \cref{lemma:extremal_lower_bound} and a matching upper bound relying on \cref{thm:retraction}.
The classification of $\LR(\cE)$ follows as a direct consequence of \cref{thm:SCD_LR_equivalence}. Before presenting the proofs, we give a brief overview and discuss why the binary cube $\bcube$ must be treated as a separate case in both of our bounds.

\subsection{Discussion and overview of the proof}

The cubical complex $\Gamma_\cE$ of a class $\cE$ is composed of the union of all its cubes, where a cube of dimension $k$ corresponds to a strongly shattered set of size $k$.  
By \cref{theorem:LC_of_unions},  $\Gamma_\cE$ has topological dimension equal to that of any maximum cube, which in the case of extremal $\cE$ is exactly $\VCdim(\cE)$.

\begin{proposition}[Dimension of the cubical complex]\label{prop:dimension_of_ccs}
For any concept class $\cC$, the topological dimension of its cubical complex $\Gamma_\cC$ is equal to the size of the largest strongly shattered set by $\cC$. In particular, if $\cE$ is an extremal class, then
\[
    \dim(\Gamma_\cE) = \VCdim(\cE).
\]
\end{proposition}

In view of \cref{prop:dimension_of_ccs}, the lower bound of \cref{thm:intro_LCdim_extremal} can be achieved by embedding $\Gamma_\cE$ into $\Delta_\cE$ and applying Lebesgue's covering theorem (\cref{thm:lebesgue}).
The upper bound follows from a retraction from $\Delta_{\cE, 1}$ to the embedding of $\Gamma_{\cE, 1}$, which we denoted by $\widetilde{\Gamma}_{\cE, 1}$ in \eqref{eq:embedding_image}. In particular, each vertex $h \in \Delta_{\cE, 1}$ is mapped to a point in a simplex of $\widetilde{\Gamma}_{\cE,1}$ supported on concepts extending $h$.
The topological dimension of $\Gamma_\cE$ is used to construct a suitable open cover of order $\VCdim(\cE)$, which is pulled back to $\Delta_\cE$ without increasing the order.

When $\cE = \bcube$, however, the embedding given in \cref{lemma:full_subcomplex} fails. In fact, $\Gamma_{\bcube}$ is the geometric cube of dimension $|\cX|$, whereas $\Delta_{\bcube}$ is the cross-polytope of dimension $|\cX|-1$, so no one-to-one embedding exists. We can recover our argument by instead considering the boundary of $\Gamma_{\bcube}$, which is homeomorphic to $\Delta_{\bcube}$, and has dimension one less than the cube.
\[
    \dim(\partial \Gamma_{\bcube}) = \VCdim(\bcube) - 1.
\]
This modification leads to the off-by-one case in \cref{thm:intro_LCdim_extremal}.

\subsection{Lower bound for simplicial covering dimension of extremal classes}

\begin{lemma}\label{lemma:extremal_lower_bound}
    Let $\cE \subseteq \bcube$ be an extremal class. Then
    \[
        \SCdim(\cE) \ge
        \begin{cases}
            \VCdim(\cE)-1 & \text{ if }\cE=\set{\pm 1}^\cX  \\
            \VCdim(\cE) &\text{otherwise.} 
        \end{cases}
    \]
\end{lemma}

\begin{proof}
    If $\cE$ is not the binary cube $\bcube$, let $Q$ be a largest cube within the cubical complex $Q_\cE$. If $\cE = \bcube$, then instead take $Q$ to be an arbitrary $(\VCdim(\cE)-1)$-dimensional cube in $Q_\cE$. Let $\Pi \subseteq \Gamma_\cE$ be the geometric realization of $Q$.
    By \Cref{prop:dimension_of_ccs}, $\dim(\Pi)=\VCdim(\cE)$ if $\cE \neq \bcube$ and $\dim(\Pi)=\VCdim(\cE)-1$  if $\cE = \bcube$.

    Since $f_\cE$, defined in~\eqref{eq:embedding}, is an injection, any cover $\cB$ of $\Delta_\cE$ induces a cover $g(\cB) = f^{-1}_{\cE}(\cB) \cap \Pi$ of $\Pi$ with $\ord(\cB) \geq \ord(g(\cB))$.
    Furthermore, any refinement of $\cB$ induces a refinement of $g(\cB)$.
    We will construct a cover $\cA$ of $\Delta_\cE$ which is refined by the zero-loss sets $\set{B_h}_{h \in \set{\pm 1}^\cX}$ and such that $g(\cA)$ fulfills the conditions for \cref{thm:lebesgue}.
    Thus, any refinement of $\cA$ must have order at least $d$.
    
    Let $\alpha < 1 / |\cX|$. We define the open cover $\cA$ of $\Delta_\cE$ as follows.
    \begin{align*}
        \cA = \set{B_h^{(\alpha)}:~h \in \set{\pm 1}^\cX}.
    \end{align*}
    It is clear that $\cA$ is refined by the zero-loss sets $\set{B_h}_{h \in \set{\pm 1}^\cX}$.

    Let $y$ lie on some face of $\Pi$, so that without loss of generality $y_x = 1$ for some $x \in \cX$, and let $h$ be any hypothesis in $\set{\pm 1}^\cX$ such that $h(x) = -1$.
    Set $y' = f_\cE(y)$, so that $y'_x \geq \frac{1}{|\cX|}$.
    Since any point $z \in B_h$ has $z_x \leq 0$, $y'$ is too far from $B_h$ to lie in $B_h^{(\alpha)}$.
    \begin{align*}
        \min_{z \in B_h} \norm{y'-z}_1 &\geq \min_{z \in B_h} |y'_x - z_x|
        \geq |y'_x| \geq \frac{1}{|\cX|} > \alpha.
    \end{align*}
    Thus, if $y$ and $z$ are two points on opposite faces of $\Pi$, neither can lie in the same set $g(B_h^{(\alpha)})$.
    This means that $g(\cA)$ satisfies the conditions of \cref{thm:lebesgue}, and we are done.
\end{proof}

This lower bound requires $\cE$ to be extremal, as we require the largest shattered set to form a solid cube in the cubical complex.

\subsection{Upper bound for simplicial covering dimension of extremal classes}

The upper bound utilizes a ``topological equivalence'' between the simplicial complex of realizable distributions $\Delta_{\cE}$ and the cubical complex of strongly shattered sets $\Gamma_{\cE}$. This equivalence is made precise by the following definition. A \emph{retraction} from a topological space $X$ to a subspace $A$ is a continuous mapping $r\colon X\rightarrow A$ such that the restriction of $r$ to $A$ is the identity, i.e. $r(a)=a$ for all $a\in A$. 

The following theorem is related to \cite[Theorem 9]{chornomaz2025spherical}, and the proof uses similar ideas, particularly the inductive construction of retractions. However, both the theorem and its proof include new elements and techniques that place the result in the framework of simplicial covering dimension. For instance, in \cite{chornomaz2025spherical}, the retraction is mainly concerned with preserving antipodality, while in our case, the retraction is designed to produce a suitable open cover of $\Delta_{\cE}$, which requires subtle analysis.  

\begin{theorem}\label{thm:retraction}
    Let $\cE\neq \{\pm1\}^{\cX}$ be an extremal class, and let $\epsilon_0>0$. Then there is a retraction $f$ from $\Delta_{\cE}$ to $\widetilde{\Gamma}_{\cE}$ and an open cover $\set{U_g}_{g\in \cE}$ of $\widetilde{\Gamma}_{\cE}$ such that the collection $\{f^{-1}(U_{g}\cap \widetilde{\Gamma}_{\cE})\}_{g\in \cE}$ is an open cover of $\Delta_{\cE}$ of order $\VCdim(\cE)$ and  $f^{-1}(U_g\cap \widetilde{\Gamma}_{\cE})\subseteq \sigma_g^{(\epsilon_0)}$ for all $g\in \cE$, where $\sigma_g$ is the maximal simplex in $\Delta_{\cE}$ corresponding to the concept $g$.
\end{theorem}

We defer the proof of \cref{thm:retraction} to the following subsections in order to focus immediately on its consequences. Namely, the theorem provides an upper bound on the simplicial covering dimension of $\cE$ that is needed to complete the proof of our main result.

Recall that  $\LR_\cH(\cC)$ denotes the list replicability number of $\cC$, with outputs in $\cH$. Similarly, we have $\SCdim_{\cH}(\cC)\coloneq\SCdim_{\Delta_\cH}(\Delta_\cC)$.

\begin{theorem}[General form of \Cref{thm:intro_LCdim_extremal}]\label{thm: general form of THM A}
     Let $\cE \subseteq \set{\pm1}^{\cX}$ be an extremal class. Then for every  $\cE \subseteq \cH \subseteq \set{\pm1}^{\cX}$, we have 
     \[
        \SCdim_{\cH}(\cE) =
        \begin{cases}
            \VCdim(\cE)-1 & \text{ if }\cE=\set{\pm 1}^\cX  \\
            \VCdim(\cE) &\text{otherwise} 
        \end{cases}.
    \]

\end{theorem}  
\begin{proof}
Note that $\SCdim(\cE) \leq \SCdim_{\cH}(\cE) \leq \SCdim_{\cE}(\cE)$ for any $\cH \supseteq \cE$. Hence, to determine $\SCdim_{\cH}(\cE)$, we bound $\SCdim(\cE)$ from below and $\SCdim_{\cE}(\cE)$ from above. Immediately by \cref{lemma:extremal_lower_bound}, we have the lower bounds
\[
        \SCdim(\cE) \ge
        \begin{cases}
            \VCdim(\cE)-1 & \text{ if }\cE=\set{\pm 1}^\cX  \\
            \VCdim(\cE) &\text{otherwise.} 
        \end{cases}
\]

As for the upper bounds, consider first the case where $\cE=\bcube$. Hence, $\SCdim_{\cE}(\cE)=\SCdim(\cE)$. Recall that
\[
    \SCdim(\Delta_\cE) = \sup \Set{ \SCdim(\cA) : \cA \text{ is a finite open cover of } \Delta_\cE},
\] 
whereas $\SCdim(\cE)$ is defined by taking the supremum over covers $\cA$ that are refined by particular zero-loss sets.
It follows that 
\[
    \SCdim(\cE) \leq \SCdim(\Delta_\cE).
\]
Since $\cE = \bcube$, the space $\Delta_\cE$ is the full cross polytope in $\R^\cX$, which has topological dimension $|\cX|-1$. Furthermore, $\VCdim(\cE) = |\cX|$ because the binary cube shatters its entire domain. Together, these imply that
\[
    \SCdim(\cE) \leq \SCdim(\Delta_\cE) = |\cX|-1 = \VCdim(\cE)-1.
\]

For the second case, let $\cE\neq \set{\pm1}^{\cX}$. Let $\cA$ be any finite open cover of $\Delta_{\cE}$ such that $\cA \prec \set{\sigma \in \Delta_{\cE}}$. By \cref{lemma:metric_refinement} there exists an $\alpha>0$ such that $\cA \prec \set{\sigma^{(\alpha)}}_{\sigma \in \Delta_{\cE}}$. This implies $\cA \prec \set{\sigma_g^{(\alpha)} :~ g\in \cE}$ where $\{\sigma_g:g\in\cE\}$ are the maximal simplices of $\Delta_{\cE}$. Pick $\epsilon_0<\alpha$ and let $f$ and $\{U_g\}_{g \in \cE}$ be as in \cref{thm:retraction}. Then $ \{f^{-1}(U_g\cap \widetilde{\Gamma}_{\cE})\}_{g\in \cE}$  is an open cover of $\Delta_{\cE}$ of order $\VCdim(\cE)$ that refines $\cA$, whereby we conclude that $\SCdim_{\cE}(\cE) \leq \VCdim(\cE)$.
\end{proof}
As a consequence, we can exactly classify the proper ($\cH=\cE$) and improper list replicability number $\LR_{\cH}(\cE)$ of extremal concept classes:
\begin{corollary}
 Let $\cE \subseteq \set{\pm1}^{\cX}$ be an extremal class. Then for every $\cH \supseteq \cE$, we have 
     \[
        \LR_{\cH}(\cE) =
        \begin{cases}
            \VCdim(\cE) & \text{ if }\cE=\set{\pm 1}^\cX  \\
            \VCdim(\cE)+1 &\text{otherwise} 
        \end{cases}.
    \]
\end{corollary}
\begin{proof}
    The corollary follows directly by combining \cref{thm: general form of THM A} and \cref{thm:SCD_LR_equivalence}.
\end{proof}
\subsubsection{Toolkit}
Before proving \cref{thm:retraction}, we build up some necessary machinery, starting with some standard topology. For more detail, the reader is invited to consult~\cite{munkres2000topology,MatousekBook}.

\begin{theorem}[Universal property of the quotient topology]\label{thm:universal_property}
Let $X$ be a topological space, $\sim$ an equivalence relation on $X$, and
$q:X\to X/{\sim}$ the canonical surjection to the quotient space. Then for any topological space $Y$ and any continuous map $f:X\to Y$ for which $x\sim x'$ implies $f(x)=f(x')$,
there exists a unique continuous map $\tilde f: X/{\sim}\to Y$ such that
\[
f=\tilde f\circ q.
\]
\begin{center}
\begin{tikzcd}
X \arrow[dr, "f"'] \arrow[r, "q"] & X/{\sim} \arrow[d, dashed, "\tilde f"] \\
& Y
\end{tikzcd}
\end{center}

\end{theorem}

\begin{definition}[Contractible]
   A topological space $X$ is contractible if there is an element $x\in X$ and a continuous contraction map $F\colon X\times [0, 1] \rightarrow X$ such that $F(\cdot,0)$ is the identity map and $F(\cdot, 1) \equiv x$.
\end{definition}

\begin{definition}[Geometric cone]
    Let $X\subseteq\mathbb{R}^n$, and suppose $p\in \mathbb{R}^n$ is a point such that any line through $p$ intersects $X$ at most one point. Then the geometric cone over $X$ with vertex point $p$ is
\[C(X,p)\coloneq\{\alpha p+(1-\alpha)x:~\alpha \in [0,1], x\in X \}.\] 
\end{definition}
Note that the cone over a simplex $\sigma$ is itself a simplex with dimension one higher. 

The following is a routine exercise in topology showing the connection between a retraction of a cone to its base and the contractibility of the base. 
\begin{lemma}[Retractions of cones]\label{lemma:retraction_of_cone}
 Let $X\subseteq \mathbb{R}^n$ be contractible to $x_0 \in X$ with contraction map $F$. Let $C(X,p$) be a geometric cone over $X$ with vertex point $p$. Then the map $\widetilde{F}\colon C(X,p)\rightarrow X$ defined by $\widetilde{F}(\alpha p+(1-\alpha)x)\coloneq F(x,\alpha)$ is a retraction.
\end{lemma} 
\begin{proof}
    Notice $C(X,p) \cong (X \times [0,1]) / \sim$ where $\sim$ is the equivalence relation generated by $(x,1)=(y,1)$ for all $x,y \in X$. Since $F(x,1)=F(y,1)=x_0$ for all $x,y\in X$, by the universal property of quotient spaces, \cref{thm:universal_property}, the map $\widetilde{F}$ is continuous. It is a retraction since $\widetilde{F}(x)=F(x,0)=x$ for all $x$.
\end{proof}

The following lemma utilizes the fact that $\Delta_{\cE,1}$ is a finite simplicial complex and will be used to prevent some unwanted intersections when constructing a suitable cover of $\Delta_{\cE,1}$.
Recall that $d_{\ell_1}(X,Y) = \inf_{x\in X, y\in Y}d_{\ell_1}(x,y)$.

\begin{lemma}\label{distance subcomplexes}There exists $\gamma_0>0$ such that for any disjoint subcomplexes $X$ and $Y$ of $\Delta_{\cE, 1}$, we have \[d_{\ell_1}(X,Y)\geq \gamma_0.\] 
\end{lemma}
\begin{proof}
    Any disjoint subcomplexes $X$ and $Y$ are nonempty, disjoint compact sets, so $d_{\ell_1}(X,Y)$ is positive. Now take $\gamma_0$ to be the minimum of $d_{\ell_1}(X,Y)$ over the finitely many disjoint subcomplexes $X$ and $Y$.
\end{proof}

We proceed with some properties of extremal classes.

\begin{definition}[$h$-restrictions \cite{chornomaz2025spherical}]  Let $\cE\subseteq \{\pm1\}^{\cX}$ be an extremal class and $h$ be a partial concept on $X$. The \emph{$h$-restriction} of $\cE$, denoted $\cE_h$, is the subclass of $\cE$ consisting of all concepts that extend $h$:
\[\cE_h\coloneq \{g\in \cE:~g\geq h \}.\]   
\end{definition}
Note $\cE_h$ is non-empty if and only if $h$ is realizable by $\cE$. Next, we consider the vertex set of the subdivision $Q_{\cE_h,1}$.

\begin{lemma}[{\cite[Appendix D4]{chornomaz2025spherical}}]  \label{lemma:cubical vertex span} Let $\cE$ be extremal and $h$ be a  realizable partial concept. Then $Q_{\cE_h,1}$ is a full subcomplex of $Q_{\cE,1}$ with vertex set $V(Q_{\cE_h,1})=\{g\in V(Q_{\cE,1}):~g\geq h\}$. Moreover, if $h\geq h'$, then $Q_{\cE_{h},1}$ is a subcomplex of $Q_{\cE_{h'},1}$ and so $\Gamma_{\cE_{h},1}$ embeds into $\Gamma_{\cE_{h'},1}$.
\end{lemma}
\begin{proof} 
For the first statement, it is easy to verify that a chain of cubes $g_1\geq \dots \geq g_k$ in $Q_{\cE}$ such that $g_k\geq h$ is also a chain of cubes in $Q_{\cE_h}$ and vice versa. The second statement follows immediately as both  $Q_{\cE_{h},1}$ and $Q_{\cE_{h'},1}$ are full subcomplexes of  $Q_{\cE,1}$.
\end{proof}
The following proposition shows that the cubical complexes of extremal classes and their $h$-restrictions are contractible. 
\begin{proposition}[{\cite[Prop. 4.12 and Thm. 3.1(6)]{chalopin2022unlabeled}}]\label{contraction proposition}Let $\cE$ be extremal. Then:
	\begin{itemize}
		\item The cubical complex $\Gamma_{\cE}$ is contractible, that is there exist $x\in \Gamma_{\cE}$ and a contraction map $F\colon \Gamma_{\cE} \times [0, 1] \rightarrow \Gamma_{\cE}$ such that $F(0, \cdot)$ is the identity map and $F(1, \cdot) \equiv x$;
		\item For any partial concept $h$ realizable by $\cE$, the class $\cE_h$ is extremal. In particular, the cubical complex $\Gamma_{\cE_h}$ is contractible.
	\end{itemize}
\end{proposition}

\subsubsection{Proof of \cref{thm:retraction}}
We will work with the subdivisions $D_{\cE,1}, \Delta_{\cE,1}$ and $Q_{\cE,1}, \Gamma_{\cE,1}$ throughout the proof. Recall from \Cref{lemma:full_subcomplex} that $Q_{\cE,1}$ is a full subcomplex of $D_{\cE,1}$, which induces a simplicial embedding $\Gamma_{\cE,1} \hookrightarrow \Delta_{\cE,1}$. Recall $\widetilde{\Gamma}_{\cE,1}$ is the image of this embedding, as defined in \eqref{eq:embedding_image}. Let $\gamma_0$ be as in \cref{distance subcomplexes} and let $\epsilon<\min(\gamma_0,\epsilon_0/2)$.

First, we construct the open cover $\{U_g\}$ of $\widetilde{\Gamma}_{\cE,1}$. Note $\cB\coloneq \{\sigma_g^{(\epsilon)}\}_{g\in \cE}$ is an open cover of $\widetilde{\Gamma}_{\cE,1}$. Since $\Gamma_{\cE}$, $\Gamma_{\cE,1}$, and $\widetilde{\Gamma}_{\cE,1}$ are homeomorphic, \cref{prop:dimension_of_ccs} implies $\dim(\widetilde{\Gamma}_{\cE,1}) = \VCdim(\cE)$. Hence, there exists an open cover $\{U_g\}_{g\in \cE}$ of $\widetilde{\Gamma}_{\cE,1}$ of order $\VCdim(\cE)$  that is a shrinkage of $\cB$.

 We proceed with constructing the retraction $f$. We say $W\subseteq V(D_{\cE, 1})= \cE^\bullet\setminus\set{h_*}$ is \emph{upwards closed} if for every partial concept $w\in W$, all partial concepts extending $w$ are also in $W$. Denote by $L_W$ the full subcomplex of $D_{\cE, 1}$, spanned by the vertex set $W$, and denote by $\Lambda_{W}$ its geometric realization inherited from $\Delta_{\cE,1}$. Recall that for each $\mu \in \Delta_{\cE,1}$, the support of $\mu$ is a chain  $\supp_{\Delta_{\cE,1}}(\mu)=h_1\geq\dots\geq h_k$
 which is the smallest simplex containing $\mu$. We denote by $h[\mu]$ the minimal partial concept in the chain $\supp_{\Delta_{\cE,1}}(\mu)$, i.e., $h[\mu]=h_k$. 
 
 The strategy of the proof is to use induction on upwards closed sets $W$ starting from $W=V(Q_{\cE, 1})$. In particular, for each $W$, we show that there exists a corresponding retraction $f_W$ from $\Lambda_W$ to $\widetilde{\Gamma}_{\cE,1}$ such that the following properties hold:
 \begin{enumerate}
     \item\label{property:1} For each $\mu \in \Lambda_{W}$, we have $f_W(\mu)\in \widetilde{\Gamma}_{\cE_{h[\mu]},1}$;
     \item \label{property:2} The collection $\{f_W^{-1}(U_{g}\cap \widetilde{\Gamma}_{\cE,1})\}_{g\in \cE}$ is an open cover of $\Lambda_W$ of order $\VCdim(\cE)$, and  $f_W^{-1}(U_g\cap \widetilde{\Gamma}_{\cE,1})\subseteq \sigma_g^{(\epsilon_0)}$ for any $g\in \cE$.
 \end{enumerate}
  Finally, setting $W=V(D_{\cE, 1})$ completes the proof. 
\\~\\
\emph{Base case}: Let $W = V(Q_{\cE, 1})$. By \cref{lemma:full_subcomplex}, we have $L_{V(Q_{\cE, 1})} = Q_{\cE, 1}$, and consequently $\Lambda_{V(Q_{\cE, 1})}=\widetilde{\Gamma}_{\cE,1}$. Let $f_{V(Q_{\cE, 1})} \colon \Lambda_{V(Q_{\cE, 1})} \rightarrow \widetilde{\Gamma}_{\cE,1}$ be the identity map $f_{V(Q_{\cE, 1})}: \mu \mapsto \mu$. \cref{property:1} holds immediately by \cref{lemma:cubical vertex span}. Moreover, note that \[f^{-1}_{V(Q_{\cE, 1})}(U_g\cap \widetilde{\Gamma}_{\cE,1})=U_g\cap \widetilde{\Gamma}_{\cE,1}\subseteq \sigma_g^{(\epsilon)}\subseteq \sigma_g^{(\epsilon_0)}\]
where the first inclusion follows from the construction of $\set{U_g}_{g\in \cE}$ and the second from our choice of $\epsilon < \epsilon_0/2$. Hence, \Cref{property:2} also holds, completing the base case.
\\~\\
\emph{Induction step}: Let $W \supsetneq V(Q_{\cE,1})$, and choose $w \in W \setminus V(Q_{\cE,1})$ to be minimal, that is, $w$ does not extend any other partial concept in $W$. Define $W' = W \setminus \set{w}$. By the induction hypothesis, we assume the existence of $f_{W'}$, and our goal is to construct $f_W$.

We now introduce some standard notions from algebraic topology that describe the local structure of a simplicial complex near a vertex~\cite{bryant2002piecewise}.
For a vertex $w$ of $\Lambda_W$, the \emph{star} and \emph{link} of $w$ in $\Lambda_W$ are the subcomplexes
\[ 
\St_W(w) = \{\sigma \in \Lambda_W:~\exists
\tau \in \Lambda_W\text{~s.t.~} \sigma,w \subseteq \tau\} \ \text{ and } \ \Lk_W(w) = \{\sigma \in \St_W(w):~w\notin \sigma\}.\] 
The star consists of all simplices of $\Lambda_W$ that contain $w$, together with all of their faces, while the link consists of those simplices in the star that do not contain $w$.

\input{fig_link_star}

It is well known that the simplices of the star are precisely those of the link or of the link coned with $w$, namely,
\begin{equation}\label{eq:star of w}
\St_W(w) = \Lk_W(w) \sqcup \set{C(\sigma, w) :~ \sigma \in \Lk_W(w)},
\end{equation}
where $C(\sigma, w)$ denotes the geometric cone over $\sigma$ with apex $w$, which is itself a simplex. 

Since the link $\Lk_W(w)$ does not contain $w$, it is a subcomplex of $\Lambda_{W'}$. Consequently, 
\[\Lk_W(w) = \Lambda_{W'} \cap \St_W(w) \ \text{ and } \ \Lambda_W = \Lambda_{W'} \cup \St_W(w).\] 
Finally, note that $\widetilde{\Gamma}_{\cE,1}\cap\Lk_W(w)=\widetilde{\Gamma}_{\cE_w,1}$, by \Cref{lemma:cubical vertex span} and the minimality of $w$ in $W$. 

Observe that $\St_W(w)$ naturally inherits the geometry of a cone. Indeed, by \eqref{eq:star of w} we have

\[ \St_W(w)=\{\alpha w+(1-\alpha) \mu': \alpha \in[0,1], \mu' \in \Lk_{W}(w) \}\]
which is precisely the geometric cone $C(\Lk_W(w), w)$. Consequently, \[\Lambda_W=\Lambda_{W'}\cup C(\Lk_W(w),w).\]

The retraction $f_W$ will be defined as the composition of two intermediate retractions 
 \[ p_W\colon \Lambda_{W'}\cup C(\Lk_W(w),w)\to \widetilde{\Gamma}_{\cE,1} \cup C(\widetilde{\Gamma}_{\cE_{w,1}},w) \ \text{ and } \ q_W\colon  \widetilde{\Gamma}_{\cE,1} \cup C(\widetilde{\Gamma}_{\cE_w,1},w)\to \widetilde{\Gamma}_{\cE,1}.\] 

We begin by defining
\[p_W(\mu) \coloneqq \begin{cases}
    f_{W^{'}}(\mu) & \text{if~} \mu\in \Lambda_{W^{'}} \\
    \alpha w+(1-\alpha)f_{W^{'}}(\mu') & \text{if~} \mu= \alpha w+(1-\alpha)\mu'\in C(\Lk_W(w),w)
\end{cases}\]

Firstly, note that $\Lk_W(w) = \Lambda_{W'} \cap C(\Lk_W(w),w)$, which shows that $p_W$ is well-defined. Next, we verify that $p_W$ has the correct image: 
\begin{itemize}
\item If $\mu\in \Lambda_{W'}$, then by the induction hypothesis $\Im(f_{W'})\subseteq\widetilde{\Gamma}_{\cE,1}$. 
\item If $\mu=\alpha w + (1-\alpha) \mu' \in C(\Lk_W(w),w)$, then 
by \cref{property:1}, $f_{W'}(\mu')\in \widetilde{\Gamma}_{\cE_{h[\mu']},1}$, where recall that $h[\mu']$ is the minimal partial concept in the support of $\mu'$. Since $w$ is minimal in $W$, we have $h[\mu']\geq w$. By \cref{lemma:cubical vertex span} this implies $f_{W'}(\mu')\in \widetilde{\Gamma}_{\cE_w,1}$. Hence, the image of $p_W$ restricted to $C(\Lk_W(w),w)$ is contained in $C(\widetilde{\Gamma}_{\cE_w,1},w)$. 
\end{itemize}

Note $p_W$ is continuous on $\Lambda_{W'}$ by induction and is continuous on $C(\Lk_W(w),w)$ by the universal property of quotient spaces,  \cref{thm:universal_property}. Since both sets are closed, $p_W$ is continuous. Finally, $p_W$ is a retraction as $f_{W'}$ is the identity map on $\widetilde{\Gamma}_{\cE,1}$ and, hence, on $\widetilde{\Gamma}_{\cE_w,1}=\widetilde{\Gamma}_{\cE,1}\cap\Lk_W(w)$.

We proceed with the construction of the retraction 
\[q_W\colon \widetilde{\Gamma}_{\cE,1} \cup C(\widetilde{\Gamma}_{\cE_w,1},w)\rightarrow \widetilde{\Gamma}_{\cE,1}.\]

By~\cref{contraction proposition}, there exists a contraction map $F_w \colon \widetilde{\Gamma}_{\cE_w,1}\times [0,1]\rightarrow \widetilde{\Gamma}_{\cE_w,1}$ to some $\mu_w\in \widetilde{\Gamma}_{\cE_w,1}$. We may, if necessary, modify the contraction map and assume, without loss of generality, that
\begin{equation}\label{eq:modified contraction}F_w(\mu',\alpha)=\mu'  \ \ \forall \mu'\in \widetilde{\Gamma}_{\cE_w,1} \text{~and~} \alpha \in [0,1-\epsilon].
\end{equation}
Define
\[q_W(\mu) \coloneqq \begin{cases}
    \mu & \text{if~} \mu \in \widetilde{\Gamma}_{\cE,1}\\
    F_w(\mu',\alpha) & \text{if~} \mu= \alpha w+(1-\alpha)\mu'\in C(\widetilde{\Gamma}_{\cE_w,1},w)
\end{cases}.\]
Since $F_w(\mu',0)=\mu'$, the map $q_W$ is well-defined on $\widetilde{\Gamma}_{\cE,1}\cap C(\widetilde{\Gamma}_{\cE_w,1},w)=\widetilde{\Gamma}_{\cE_w,1}$. Moreover, as $\Im(F_w)\subseteq \widetilde{\Gamma}_{\cE_w,1}$, the image of $q_W$ is contained in $\widetilde{\Gamma}_{\cE,1}$ as desired. By \cref{lemma:retraction_of_cone}, $F_w$ is continuous on $C(\widetilde{\Gamma}_{\cE_w,1},w)$, and therefore, $q_W$ is also continuous. Finally, $q_W$ is a retraction as it acts as the identity on $\widetilde{\Gamma}_{\cE,1}$.

Since the composition of two retractions is also a retraction, we conclude that $f_W \coloneq q_W \circ p_W$ is a retraction from $\Lambda_W$ to $\widetilde{\Gamma}_{\cE,1}$. 

\textbf{\cref{property:1}:} By the induction hypothesis, the case $\mu \in\Lambda_{W'}$ is immediate. It remains to consider the case $\mu \coloneq \alpha w+(1-\alpha)\mu'\in C(\Lk_{W}(w),w)$. Observe
\begin{equation*}
\begin{split}
f_W(\alpha w +(1-\alpha)\mu') & =q_W\circ p_W(\alpha w +(1-\alpha)\mu')\\
& =q_W(\alpha w +(1-\alpha)f_{W'}(\mu'))\\ 
& = F_w(f_{W'}(\mu'),\alpha).
\end{split}
\end{equation*}
Hence, $f_W(\mu)\in \widetilde{\Gamma}_{\cE_w,1}$. Since every $\mu \in C(\Lk_{W}(w),w)$ satisfies $h[\mu]\geq w$, the property follows.

\textbf{\cref{property:2}:} Since $f_W$ is continuous, it follows immediately that the collection
$\{f_W^{-1}(U_{g}\cap \widetilde{\Gamma}_{\cE,1})\}_{g\in \cE}$ is an open cover of $\Lambda_W$ of order $\VCdim(\cE)$. 

We now show that for every $g \in \cE$,  $f_W^{-1}(U_g\cap \widetilde{\Gamma}_{\cE,1})\subseteq \sigma_g^{(\epsilon_0)}$. Fix $g\in \cE$ and suppose $f_W(\mu)\in U_g$. If $\mu\in \Lambda_{W'}$, then $f_W(\mu)=f_{W'}(\mu)\in U_g$, and, by the induction hypothesis,  $\mu\in \sigma_g^{(\epsilon_0)}$.

It remains to prove the same for $\mu \in C(\Lk_W(w),w)$. To do this, we use the following claim
\begin{claim}\label{claim intersection}
Pick any $g \in \cE$. Then $U_g \cap \widetilde{\Gamma}_{\cE_w,1}\neq \emptyset$ implies that $\sigma_g\cap \widetilde{\Gamma}_{\cE_w,1}\neq \emptyset$. In particular, $g$ extends $w$, and $w \in \sigma_g$.
\end{claim}
\begin{proof}
    We prove the contrapositive. Assume that $\sigma_g$ and $\widetilde{\Gamma}_{\cE_w,1}$ are disjoint. Then by \cref{distance subcomplexes} we have $d_{\ell_1}(\sigma_g, \widetilde{\Gamma}_{\cE_w,1}) \geq \gamma_0$. Recall $U_g\subseteq \sigma_g^{(\epsilon)}\subseteq \sigma_g^{(\gamma_0)}$ where the second containment follows from our choice of $\epsilon<\gamma_0$. Hence, $U_g$ and $\widetilde{\Gamma}_{\cE_w,1}$ are disjoint. 
    
    Finally, $\sigma_g\cap \widetilde{\Gamma}_{\cE_w,1}\neq \emptyset$ implies that these two simplicial complexes have a common vertex $v$. All vertices in the barycentric subdivision of $\sigma_g$ are extended by $g$. Hence, by \cref{lemma:cubical vertex span} we have $g\geq v \geq w$, and $w\in\sigma_g$.
\end{proof}
Now let $\mu \coloneq \alpha w + (1-\alpha)\mu' \in C(\Lk_W(w),w)$ and suppose $f_W(\mu)\in U_g$. By \cref{property:1}, $f_W(\mu)\in U_g \cap \widetilde{\Gamma}_{\cE_w,1}$. By \cref{claim intersection}, we have $w\in \sigma_g$. This implies
\[d_{\ell_1}(\alpha w+(1-\alpha)\mu',\sigma_g)\leq (1-\alpha) d_{\ell_1}(\mu',\sigma_g)\]
where we used the convexity of $\sigma_g$. It remains to bound $(1-\alpha)d_{\ell_1}(\mu',\sigma_g)$. We consider two cases:

\emph{First case}: Let $\alpha\in [0,1-\epsilon]$. By \cref{eq:modified contraction} we have 
\[f_{W}(\mu)=F_w(f_{W'}(\mu'),\alpha)=f_{W'}(\mu')\]
Since $f_{W'}(\mu')\in U_g$, by the induction hypothesis, we have that $\mu'\in \sigma_g^{(\epsilon_0)}$. Hence, 

\[d_{\ell_1}(\alpha w+(1-\alpha)\mu',\sigma_g)\leq (1-\alpha) d_{\ell_1}(\mu',\sigma_g)< 1\times\epsilon_0=\epsilon_0.\] 

\emph{Second case}: Let $\alpha\in (1-\epsilon,1]$. Then
\[d_{\ell_1}(\alpha w+(1-\alpha)\mu',\sigma_g)\leq (1-\alpha) d_{\ell_1}(\mu',\sigma_g)< \epsilon\times2<\epsilon_0.\]
where we used the fact that the maximum $\ell_1$-distance between two points of $\Delta_{\cE}$ is two, together with our choice of $\epsilon<\epsilon_0/2$.

Both cases are complete. Hence, $\mu \in\sigma_g^{(\epsilon_0)}$ and the induction is complete.

\bibliographystyle{alphaurl}  
\bibliography{refs}
\appendix

\end{document}

%% file: header.tex
\usepackage{amsmath,amsthm,amssymb,amscd,amstext,amsfonts}
\usepackage{color,verbatim,graphicx,fullpage,url}
\usepackage{xcolor}
\usepackage{algorithm}
\usepackage{algorithmic}
\usepackage{nicefrac}
\usepackage{fullpage,tikz,enumitem}
\usepackage{tikz-cd}
\usepackage{bm}
\usepackage{physics}
\usepackage{todonotes}
\usepackage{mathtools}
\usepackage{thm-restate}

\usepackage[breaklinks=true, linktocpage=true,]{hyperref}
\hypersetup{
    colorlinks = true,
    linkcolor={purple},
    urlcolor={purple},
    citecolor={blue!80!black}
}

\usepackage[nameinlink, noabbrev, capitalize]{cleveref}

\usepackage[utf8]{inputenc}
 
\Crefname{enumi}{Property}{Properties}

\theoremstyle{plain}

\newtheorem{thm}{Theorem}[section]
\newtheorem{theorem}[thm]{Theorem}
\newtheorem{atheorem}{Theorem}

\crefname{atheorem}{Theorem}{Theorems}
\Crefname{atheorem}{Theorem}{Theorems}

\newtheorem{corollary}[thm]{Corollary}
\newtheorem{lemma}[thm]{Lemma}
\newtheorem{proposition}[thm]{Proposition}
\newtheorem{definition}[thm]{Definition}
\newtheorem{example}[thm]{Example}

\newtheorem{question}[thm]{Question}
\newtheorem{claim}[thm]{Claim}
\newtheorem*{claim*}{Claim}

\theoremstyle{remark}
\newtheorem{remark}[thm]{Remark}

\renewcommand\norm[1]{\left|\!\left|#1\right|\!\right|}

\newcommand{\set}[1]{\{#1\}}
\newcommand{\Set}[1]{\left\{#1\right\}}

\newcommand{\inp}[2]{{\left\langle #1,#2 \right\rangle}}            

\DeclareMathOperator{\Prob}{Pr}                        

\newcommand{\R}{\mathbb{R}}

\newcommand{\bZ}{\mathbb{Z}}

\newcommand{\cB}{\mathcal B}
\newcommand{\cF}{\mathcal F}

\newcommand{\mcV}{\mathcal V}
\newcommand{\cX}{\mathcal X}

\newcommand{\bS}{\mathbf S}

\usepackage{ifthen}
\newcommand{\agnorm}[2][]{
	\ifthenelse{\equal{#2}{}}{
		\widetilde{\gamma}_2^{#1}
	}{
		\widetilde{\gamma}_2^{#1}(#2)
	}
}

\DeclareFontFamily{U}{mathx}{}
\DeclareFontShape{U}{mathx}{m}{n}{<-> mathx10}{}
\DeclareSymbolFont{mathx}{U}{mathx}{m}{n}
\DeclareMathAccent{\widecheck}{0}{mathx}{"71}

\newcommand{\cC}{\mathcal{C}}
\newcommand{\cH}{\mathcal{H}}
\newcommand{\cA}{\mathcal{A}}

\newcommand{\cE}{\mathcal{E}}

\newcommand{\cQ}{\mathcal{Q}}

\newcommand{\cU}{\mathcal{U}}

\DeclareMathOperator{\VCdim}{\textsc{vc}}
\DeclareMathOperator{\SCdim}{\textsc{sc}}
\DeclareMathOperator{\LR}{\textsc{lr}}
\DeclareMathOperator{\LC}{L}
\DeclareMathOperator{\SDdim}{\textsc{sd}}

\DeclareMathOperator{\sign}{sign}
\DeclareMathOperator{\ord}{ord}

\renewcommand{\cH}{\mathcal{H}}
\renewcommand{\bS}{\mathbb{S}}

\renewcommand{\Pr}{\mathbb{P}}

\newcommand{\St}{\mathrm{St}}
\newcommand{\Lk}{\mathrm{Lk}}

\DeclareMathOperator{\conv}{conv}                                             
\DeclareMathOperator{\supp}{supp}                                             
\DeclareMathOperator{\loss}{loss}                                             
\DeclareMathOperator{\signrank}{sign-rank}                                     

\let\emptyset\varnothing

\newcommand{\bcube}{{\{ \pm 1 \}^{\cX}}}

%% file: fig_lebesgue.tex
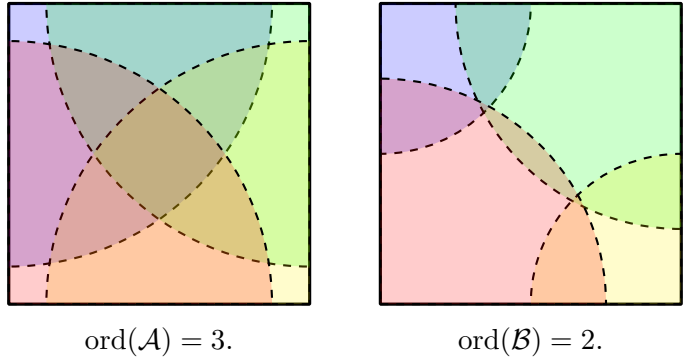
\begin{figure}[H]
    \begin{center}
        \begin{tikzpicture}[baseline=(current bounding box.center),scale=1]
            \coordinate (A) at (-2, 2);
            \coordinate (B) at (2, 2);
            \coordinate (C) at (2, -2);
            \coordinate (D) at (-2, -2);
            \draw[very thick, rounded corners=0.1mm] (A) -- (B) -- (C) -- (D) -- cycle;

            \filldraw[thick, dashed, fill=blue, fill opacity=0.2] (A) -- +(3.5,0) arc[start angle=0, delta angle=-90, radius=3.5] -- (A);
            \filldraw[thick, dashed, fill=green, fill opacity=0.2] (B) -- +(-3.5,0) arc[start angle=180, delta angle=90, radius=3.5] -- (B);
            \filldraw[thick, dashed, fill=yellow, fill opacity=0.2] (C) -- +(-3.5,0) arc[start angle=180, delta angle=-90, radius=3.5] -- (C);
            \filldraw[thick, dashed, fill=red, fill opacity=0.2] (D) -- +(3.5,0) arc[start angle=0, delta angle=90, radius=3.5] -- (D);
            \node at (0, -2.5) {$\ord(\cA) = 3$.};
        \end{tikzpicture}
        \quad \quad
        \begin{tikzpicture}[baseline=(current bounding box.center),scale=1]
            \coordinate (A) at (-2, 2);
            \coordinate (B) at (2, 2);
            \coordinate (C) at (2, -2);
            \coordinate (D) at (-2, -2);
            \draw[very thick, rounded corners=0.1mm] (A) -- (B) -- (C) -- (D) -- cycle;
    
            \filldraw[thick, dashed, fill=blue, fill opacity=0.2] (A) -- +(2,0) arc[start angle=0, delta angle=-90, radius=2] -- (A);
            \filldraw[thick, dashed, fill=green, fill opacity=0.2] (B) -- +(-3,0) arc[start angle=180, delta angle=90, radius=3] -- (B);
            \filldraw[thick, dashed, fill=yellow, fill opacity=0.2] (C) -- +(-2,0) arc[start angle=180, delta angle=-90, radius=2] -- (C);
            \filldraw[thick, dashed, fill=red, fill opacity=0.2] (D) -- +(3,0) arc[start angle=0, delta angle=90, radius=3] -- (D);
            \node at (0, -2.5) {$\ord(\cB) = 2$.};
        \end{tikzpicture}
    \end{center}
        
    \caption{The illustrated open cover $\cA$ of the square $[-1,1]^2$ has order $3$, but it admits a shrinkage $\cB$ of order $2$, implying $\LC(\cA) \le 2$. Moreover, since no set in $\cA$ contains points from opposite faces, the Lebesgue covering theorem shows $\LC(\cA) \ge 2$.}
    \label{fig:lebesgue_covering_dimension}
\end{figure}

%% file: fig_Delta.tex
\begin{figure}[H]
    \begin{center}
        \tikzset{
            every node/.style={
            }
        }

        {
        \begin{tikzpicture}[baseline=(current bounding box.center), scale=0.6]
            \coordinate (LL) at (-4, 0);
            \coordinate (TL) at (-1, 3);
            \coordinate (TR) at (5, 3);
            \coordinate (BR) at (5, -3);
            \coordinate (BL) at (-1, -3);
            \coordinate (M) at (2, 0);

            \filldraw[thick, rounded corners=0.3mm, fill=gray!10]
                (TR) -- (TL) -- (M) -- (TR)
                (TR) -- (BR) -- (M) -- (TR)
                (BR) -- (BL) -- (M) -- (BR)
                (BL) -- (TL) -- (M) -- (BL)
                (TL) -- (LL) -- (BL) -- (TL);


            \filldraw[shift only](TL) circle (2pt) node[anchor=south]{};
            \filldraw[shift only] (TR) circle (2pt) node[anchor=south]{};
            \filldraw[shift only] (BR) circle (2pt) node[anchor=north]{};
            \filldraw[shift only] (BL) circle (2pt) node[anchor=north]{};
            \filldraw[shift only] (LL) circle (2pt) node[anchor=east]{};
            
            \filldraw[shift only] (M) circle (2pt);
            \node at (2, 2) {\texttt{+-+}};
            \node at (2, -2) {\texttt{-++}};
            \node at (3.6, 0) {\texttt{--+}};
            \node at (0.4, 0) {\texttt{+++}};
            \node at (-2.4, 0) {\texttt{++-}};

            \node[font=\large] at (-4, 3) {$\Delta_{\cC}$};
        \end{tikzpicture}
        }
        \quad
        \begin{tikzpicture}[baseline=(current bounding box.center), scale=0.8]
            \coordinate (TL) at (-1, 3);
            \coordinate (TR) at (5, 3);
            \coordinate (BR) at (5, -3);
            \coordinate (BL) at (-1, -3);
            \coordinate (LL) at (-4, 0);

            \coordinate (tl) at (-2.5, 1.5);
            \coordinate (bl) at (-2.5, -1.5);
            \coordinate (TT) at (2, 3);
            \coordinate (RR) at (5, 0);
            \coordinate (BB) at (2, -3);

            \coordinate (A) at (-2, 0);
            \coordinate (AB) at (-1, 0);
            \coordinate (B) at (0, 0);
            \coordinate (BC) at (0.5, 1.5);
            \coordinate (C) at (2, 2);
            \coordinate (CD) at (3.5, 1.5);
            \coordinate (D) at (4, 0);
            \coordinate (DE) at (3.5, -1.5);
            \coordinate (E) at (2, -2);
            \coordinate (EB) at (0.5, -1.5);
            \coordinate (M) at (2, 0);

            \filldraw[thick, rounded corners=0.3mm, fill=gray!10]
                (TR) -- (TL) -- (M) -- (TR)
                (TR) -- (BR) -- (M) -- (TR)
                (BR) -- (BL) -- (M) -- (BR)
                (BL) -- (TL) -- (M) -- (BL)
                (TL) -- (LL) -- (BL) -- (TL);

            \filldraw[shift only]
                (TL) circle (2pt) node[anchor=south]{\texttt{+**}}
                (TT) circle (2pt) node[anchor=south]{\texttt{+-*}}
                (TR) circle (2pt) node[anchor=south]{\texttt{*-*}}
                (RR) circle (2pt) node[anchor=west]{\texttt{--*}}
                (BR) circle (2pt) node[anchor=north]{\texttt{-**}}
                (BB) circle (2pt) node[anchor=north]{\texttt{-+*}}
                (BL) circle (2pt) node[anchor=north]{\texttt{*+*}}
                (bl) circle (2pt) node[anchor=north east]{\texttt{*+-}}
                (LL) circle (2pt) node[anchor=east]{\texttt{**-}}
                (tl) circle (2pt) node[anchor=south east]{\texttt{+*-}}
                (A) circle (2pt) node[anchor=south east]{\texttt{++-}}
                (B) circle (2pt) node[anchor=south west]{\texttt{+++}}
                (C) circle (2pt) node[anchor=south]{\texttt{+-+}}
                (D) circle (2pt) node[anchor=south east]{\texttt{--+}}
                (E) circle (2pt) node[anchor=north]{\texttt{-++}}
                (AB) circle (2pt) node[anchor=north east, xshift=2pt]{\texttt{++*}}
                (BC) circle (2pt) node[anchor=east]{\texttt{+*+}}
                (CD) circle (2pt) node[anchor=west]{\texttt{*-+}}
                (DE) circle (2pt) node[anchor=west]{\texttt{-*+}}
                (EB) circle (2pt) node[anchor=east]{\texttt{*++}}
                (M) circle (2pt);
            \node at (2, 0.8) {\texttt{**+}};

            \draw[dashed]
                (TT) -- (BB)
                (LL) -- (RR)
                (bl) -- (TL) -- (CD) -- (BR) -- (EB) -- (TL)
                (tl) -- (BL) -- (BC) -- (TR) -- (DE) -- (BL);

            \node[font=\large] at (-4, 3) {$\Delta_{\cC,1}$};
        \end{tikzpicture}
    \end{center}
    \caption{The simplicial complexes $\Delta_\cC$ and $\Delta_{\cC,1}$ for $\cC = \set{\texttt{++-}, \texttt{+++}, \texttt{+-+}, \texttt{--+}, \texttt{-++}}$. Here, $\texttt{+}$ and $\texttt{-}$ are shorthand for $+1$ and $-1$, respectively.}
    \label{fig:Delta}
\end{figure}

%% file: fig_3dcube.tex
\begin{figure}[H]
    \begin{center}
        \tikzset{
            every node/.style={
            }
        }
        

    \begin{tikzpicture}[baseline=(current bounding box.center), x={(0.7cm,-0.3cm)}, y={(0cm,1cm)}, z={(0.8cm,0.3cm)}, scale=1.5]
            \coordinate (A) at (-1, -1, -1);
            \coordinate (AB) at (-1, 0, -1);
            \coordinate (B) at (-1, 1, -1);
            \coordinate (BC) at (-1, 1, 0);
            \coordinate (C) at (-1, 1, 1);
            \coordinate (CD) at (0, 1, 1);
            \coordinate (D) at (1, 1, 1);
            \coordinate (DE) at (1, 1, 0);
            \coordinate (E) at (1, 1, -1);
            \coordinate (EB) at (0, 1, -1);
            \coordinate (M) at (0, 1, 0);
            \filldraw[thick, rounded corners=0.3mm, fill=blue!20] (A) -- (B) -- (C) -- (D) -- (E) -- (B);

            \filldraw[shift only]
                (A) circle (2pt) node[anchor=east]{\texttt{++-}}
                (B) circle (2pt) node[anchor=south east]{\texttt{+++}}
                (C) circle (2pt) node[anchor=south]{\texttt{+-+}}
                (D) circle (2pt) node[anchor=west]{\texttt{--+}}
                (E) circle (2pt) node[anchor=north]{\texttt{-++}}
                ;


            \node[font=\large] at (0.7, -.7, .7) {$\Gamma_{\cC}$};
        \end{tikzpicture}
        \qquad 
        \begin{tikzpicture}[baseline=(current bounding box.center), x={(0.7cm,-0.3cm)}, y={(0cm,1cm)}, z={(0.8cm,0.3cm)}, scale=1.5]
            \coordinate (A) at (-1, -1, -1);
            \coordinate (AB) at (-1, 0, -1);
            \coordinate (B) at (-1, 1, -1);
            \coordinate (BC) at (-1, 1, 0);
            \coordinate (C) at (-1, 1, 1);
            \coordinate (CD) at (0, 1, 1);
            \coordinate (D) at (1, 1, 1);
            \coordinate (DE) at (1, 1, 0);
            \coordinate (E) at (1, 1, -1);
            \coordinate (EB) at (0, 1, -1);
            \coordinate (M) at (0, 1, 0);
            \filldraw[thick, rounded corners=0.3mm, fill=blue!20] (A) -- (B) -- (C) -- (D) -- (E) -- (B);

            \filldraw[shift only]
                (A) circle (2pt) node[anchor=east]{\texttt{++-}}
                (AB) circle (2pt) node[anchor=east]{\texttt{++*}}
                (B) circle (2pt) node[anchor=south east]{\texttt{+++}}
                (BC) circle (2pt) node[anchor=south east]{\texttt{+*+}}
                (C) circle (2pt) node[anchor=south]{\texttt{+-+}}
                (CD) circle (2pt) node[anchor=south west]{\texttt{*-+}}
                (D) circle (2pt) node[anchor=west]{\texttt{--+}}
                (DE) circle (2pt) node[anchor=north west]{\texttt{-*+}}
                (E) circle (2pt) node[anchor=north]{\texttt{-++}}
                (EB) circle (2pt) node[anchor=north east]{\texttt{*++}}
                (M) circle (2pt) node[anchor=south]{\texttt{**+}};

            \draw[dashed] 
                (B) -- (D)
                (BC) -- (DE)
                (C) -- (E)
                (CD) -- (EB);

            \node[font=\large] at (0.7, -.7, .7) {$\Gamma_{\cC, 1}$};
        \end{tikzpicture}
    \end{center}
    \caption{The cubical complex $\Gamma_{\cC}$ and the simplicial complex $\Gamma_{\cC,1}$  for $\cC = \set{\texttt{++-}, \texttt{+++}, \texttt{+-+}, \texttt{--+}, \texttt{-++}}$.}
    \label{fig:Cubical}
\end{figure}

%% file: fig_embedding.tex
\begin{figure}[H]
    \begin{center}
        \tikzset{
        }
        
        \begin{tikzpicture}[baseline=(current bounding box.center), x={(1cm,0cm)}, y={(0cm,1cm)}, scale=0.6]
            \coordinate (TL) at (-1, 3);
            \coordinate (TR) at (5, 3);
            \coordinate (BR) at (5, -3);
            \coordinate (BL) at (-1, -3);
            \coordinate (LL) at (-4, 0);

            \coordinate (tl) at (-2.5, 1.5);
            \coordinate (bl) at (-2.5, -1.5);
            \coordinate (TT) at (2, 3);
            \coordinate (RR) at (5, 0);
            \coordinate (BB) at (2, -3);

            \coordinate (A) at (-2, 0);
            \coordinate (AB) at (-1, 0);
            \coordinate (B) at (0, 0);
            \coordinate (BC) at (0.5, 1.5);
            \coordinate (C) at (2, 2);
            \coordinate (CD) at (3.5, 1.5);
            \coordinate (D) at (4, 0);
            \coordinate (DE) at (3.5, -1.5);
            \coordinate (E) at (2, -2);
            \coordinate (EB) at (0.5, -1.5);
            \coordinate (M) at (2, 0);

            \filldraw[thick, rounded corners=0.3mm, fill=gray!10]
                (TR) -- (TL) -- (M) -- (TR)
                (TR) -- (BR) -- (M) -- (TR)
                (BR) -- (BL) -- (M) -- (BR)
                (BL) -- (TL) -- (M) -- (BL)
                (TL) -- (LL) -- (BL) -- (TL);

            \filldraw[very thick, rounded corners=0.3mm, blue, fill=blue!20]
                (A) -- (B) -- (BC) -- (C) -- (CD) -- (D) -- (DE) -- (E) -- (EB) -- (B);

            \filldraw[shift only] 
                (TL) circle (2pt) node[anchor=south]{\texttt{+**}}
                (TT) circle (2pt) node[anchor=south]{\texttt{+-*}}
                (TR) circle (2pt) node[anchor=south]{\texttt{*-*}}
                (RR) circle (2pt) node[anchor=west]{\texttt{--*}}
                (BR) circle (2pt) node[anchor=north]{\texttt{-**}}
                (BB) circle (2pt) node[anchor=north]{\texttt{-+*}}
                (BL) circle (2pt) node[anchor=north]{\texttt{*+*}}
                (bl) circle (2pt) node[anchor=north east]{\texttt{*+-}}
                (LL) circle (2pt) node[anchor=east]{\texttt{**-}}
                (tl) circle (2pt) node[anchor=south east]{\texttt{+*-}}
                (M) circle (2pt) node[anchor=192, outer sep=1.7mm]{\texttt{**+}};

            \draw[dashed] 
                (B) -- (D)
                (BC) -- (DE)
                (C) -- (E)
                (CD) -- (EB);

            \draw[dashed]
                (tl) -- (BL)
                (bl) -- (TL) -- (C) -- (TR) -- (D) -- (BR) -- (E) -- (BL) -- (B) -- (TL)
                (LL) -- (A)
                (TT) -- (C)
                (RR) -- (D)
                (BB) -- (E);

            \node[font=\large] at (1, -5) {$\widetilde{\Gamma}_{\cC, 1} \coloneqq f_{\cC}(\Gamma_{\cC, 1}) \subseteq \Delta_{\cC, 1}$};
        \end{tikzpicture}
        \quad
        \begin{tikzpicture}[baseline=(current bounding box.center), x={(1cm,0cm)}, y={(0cm,1cm)}, z={(0.1cm,0.4cm)}, scale=1]
            \coordinate (A) at (-1, -1, -1);
            \coordinate (AB) at (-1.5, 0, -1.5);
            \coordinate (B) at (-1, 1, -1);
            \coordinate (BC) at (-1.5, 1.5, 0);
            \coordinate (C) at (-1, 1, 1);
            \coordinate (CD) at (0, 1.5, 1.5);
            \coordinate (D) at (1, 1, 1);
            \coordinate (DE) at (1.5, 1.5, 0);
            \coordinate (E) at (1, 1, -1);
            \coordinate (EB) at (0, 1.5, -1.5);

            \draw[very thick, rounded corners=0.1mm, fill=black, fill opacity=0.1] (0, 3, 0) -- (-3, 0, 0) -- (0, 0, -3) -- cycle;
            \draw[very thick, rounded corners=0.1mm, fill=black, fill opacity=0.1] (0, -3, 0) -- (-3, 0, 0) -- (0, 0, -3) -- cycle;
            \draw[very thick, rounded corners=0.1mm, fill=black, fill opacity=0.1] (0, 3, 0) -- (3, 0, 0) -- (0, 0, -3) -- cycle;
            \draw[rounded corners=0.1mm, fill=black, fill opacity=0.1] (0, 3, 0) -- (-3, 0, 0) -- (0, 0, 3) -- cycle;
            \draw[rounded corners=0.1mm, fill=black, fill opacity=0.1] (0, 3, 0) -- (3, 0, 0) -- (0, 0, 3) -- cycle;

            \draw[very thick, rounded corners=0.1mm, blue, fill=blue, fill opacity=0.2]
                (A) -- (AB);
            \draw[very thick, rounded corners=0.1mm, blue, fill=blue, fill opacity=0.2]
                (AB) -- (B);
            \draw[very thick, rounded corners=0.1mm, blue, fill=blue, fill opacity=0.2]
                (B) -- (BC) -- (0, 3, 0) -- (EB) -- cycle;
            \draw[thick, rounded corners=0.1mm, blue, fill=blue, fill opacity=0.2]
                (C) -- (CD) -- (0, 3, 0) -- (BC) -- cycle;
            \draw[thick, rounded corners=0.1mm, blue, fill=blue, fill opacity=0.2]
                (D) -- (DE) -- (0, 3, 0) -- (CD) -- cycle;
            \draw[very thick, rounded corners=0.1mm, blue, fill=blue, fill opacity=0.2]
                (E) -- (EB) -- (0, 3, 0) -- (DE) -- cycle;

            \node[font=\large] at (-0, -3.5, 0) {$\widetilde{\Gamma}_{\cC} \coloneqq f_\cC(\Gamma_{\cC}) \subseteq \Delta_\cC \subseteq \R^3$};
        \end{tikzpicture}
    \end{center}
    \caption{The embeddings $\Gamma_{\cC,1}\hookrightarrow \Delta_{\cC,1}$ and $\Gamma_{\cC} \hookrightarrow \Delta_{\cC}$  for the concept class $\cC = \set{\texttt{++-}, \texttt{+++}, \texttt{+-+}, \texttt{--+}, \texttt{-++}}$.} 
    \label{fig:embedding}
\end{figure}

%% file: fig_link_star.tex
\begin{figure}[H]
    \begin{center}
        \tikzset{
        }
        
        \begin{tikzpicture}[baseline=(current bounding box.center), x={(1cm,0cm)}, y={(0cm,1cm)}, scale=0.8]

            \clip (-5.5,-1) rectangle (3,4);
                    
            \coordinate (TL) at (-1, 3);
            \coordinate (TR) at (5, 3);
            \coordinate (BR) at (5, -3);
            \coordinate (BL) at (-1, -3);
            \coordinate (LL) at (-4, 0);

            \coordinate (tl) at (-2.5, 1.5);
            \coordinate (bl) at (-2.5, -1.5);
            \coordinate (TT) at (2, 3);
            \coordinate (RR) at (5, 0);
            \coordinate (BB) at (2, -3);

            \coordinate (A) at (-2, 0);
            \coordinate (AB) at (-1, 0);
            \coordinate (B) at (0, 0);
            \coordinate (BC) at (0.5, 1.5);
            \coordinate (C) at (2, 2);
            \coordinate (CD) at (3.5, 1.5);
            \coordinate (D) at (4, 0);
            \coordinate (DE) at (3.5, -1.5);
            \coordinate (E) at (2, -2);
            \coordinate (EB) at (0.5, -1.5);
            \coordinate (M) at (2, 0);

            \filldraw[rounded corners=0.3mm, fill=gray!10]
                (TR) -- (TL) -- (M) -- (TR)
                (TR) -- (BR) -- (M) -- (TR)
                (BR) -- (BL) -- (M) -- (BR)
                (BL) -- (TL) -- (M) -- (BL)
                (TL) -- (LL) -- (BL) -- (TL);

            \draw[dashed]
                (A) -- (B) -- (BC) -- (C) -- (CD) -- (D) -- (DE) -- (E) -- (EB) -- (B);

            \filldraw[rounded corners=0.3mm, very thick, blue, fill=blue!20]
                (tl) -- (A) -- (B) -- (BC) -- (C) --
                (TT) -- (TL) -- cycle;

            \draw[rounded corners=0.3mm]
                (TL) -- (M)
                (TR) -- (BR) -- (M) -- (TR)
                (BR) -- (BL) -- (M) -- (BR)
                (BL) -- (TL) -- (M) -- (BL)
                (LL) -- (BL) -- (TL);

            \filldraw[shift only, fill=blue] 
                (TL) circle (2pt) node[anchor=south]{$w \coloneqq$~\texttt{+**}}
                (TT) circle (2pt) node[anchor=south]{\texttt{+-*}}
                (TR) circle (2pt) node[anchor=south]{\texttt{*-*}}
                (RR) circle (2pt) node[anchor=west]{\texttt{--*}}
                (BR) circle (2pt) node[anchor=north]{\texttt{-**}}
                (BB) circle (2pt) node[anchor=north]{\texttt{-+*}}
                (BL) circle (2pt) node[anchor=north]{\texttt{*+*}}
                (bl) circle (2pt) node[anchor=north east]{\texttt{*+-}}
                
                (tl) circle (2pt) node[anchor=south east]{\texttt{+*-}}
                (A) circle (2pt) node[anchor=north east, outer sep=1mm]{\texttt{++-}}
                (B) circle (2pt) node[anchor=north west, outer sep=1mm]{\texttt{+++}}
                (C) circle (2pt) node[anchor=135, outer sep=2mm]{\texttt{+-+}}
                (AB) circle (2pt) node[anchor=140, outer sep=0.5mm]{\texttt{++*}}
                (BC) circle (2pt) node[anchor=172, outer sep=1mm]{\texttt{+*+}}
                ;

            \filldraw[shift only] (LL) circle (2pt) node[anchor=east]{\texttt{**-}};
            
            \draw[dashed] 
                (B) -- (D)
                (BC) -- (DE)
                (C) -- (E)
                (TR) -- (BL);

            \draw[dashed]
                (tl) -- (BL)
                (bl) -- (TL) -- (C) -- (TR) -- (D) -- (BR) -- (E) -- (BL) -- (B) -- (TL)
                (LL) -- (A)
                (TT) -- (C)
                (RR) -- (D)
                (BB) -- (E);

            \node[font=\large] at (-4, 3) {$\St_W(w)$};
        \end{tikzpicture}
        \quad \quad
        \begin{tikzpicture}[baseline=(current bounding box.center), x={(1cm,0cm)}, y={(0cm,1cm)}, scale=0.8]

            \clip (-5.5,-1) rectangle (3,4);
                    
            \coordinate (TL) at (-1, 3);
            \coordinate (TR) at (5, 3);
            \coordinate (BR) at (5, -3);
            \coordinate (BL) at (-1, -3);
            \coordinate (LL) at (-4, 0);

            \coordinate (tl) at (-2.5, 1.5);
            \coordinate (bl) at (-2.5, -1.5);
            \coordinate (TT) at (2, 3);
            \coordinate (RR) at (5, 0);
            \coordinate (BB) at (2, -3);

            \coordinate (A) at (-2, 0);
            \coordinate (AB) at (-1, 0);
            \coordinate (B) at (0, 0);
            \coordinate (BC) at (0.5, 1.5);
            \coordinate (C) at (2, 2);
            \coordinate (CD) at (3.5, 1.5);
            \coordinate (D) at (4, 0);
            \coordinate (DE) at (3.5, -1.5);
            \coordinate (E) at (2, -2);
            \coordinate (EB) at (0.5, -1.5);
            \coordinate (M) at (2, 0);

            \filldraw[rounded corners=0.3mm, fill=gray!10]
                (TR) -- (TL) -- (M) -- (TR)
                (TR) -- (BR) -- (M) -- (TR)
                (BR) -- (BL) -- (M) -- (BR)
                (BL) -- (TL) -- (M) -- (BL)
                (TL) -- (LL) -- (BL) -- (TL);

            \draw[dashed]
                (A) -- (B) -- (BC) -- (C) -- (CD) -- (D) -- (DE) -- (E) -- (EB) -- (B);

            \draw[rounded corners=0.3mm, very thick, blue]
                (tl) -- (A) -- (B) -- (BC) -- (C) --
                (TT);

            \draw[rounded corners=0.3mm]
                (TL) -- (M)
                (TR) -- (BR) -- (M) -- (TR)
                (BR) -- (BL) -- (M) -- (BR)
                (BL) -- (TL) -- (M) -- (BL)
                (LL) -- (BL) -- (TL);

            \filldraw[shift only, fill=blue]
                (TT) circle (2pt) node[anchor=south]{\texttt{+-*}}
                (TR) circle (2pt) node[anchor=south]{\texttt{*-*}}
                (RR) circle (2pt) node[anchor=west]{\texttt{--*}}
                (BR) circle (2pt) node[anchor=north]{\texttt{-**}}
                (BB) circle (2pt) node[anchor=north]{\texttt{-+*}}
                (BL) circle (2pt) node[anchor=north]{\texttt{*+*}}
                (bl) circle (2pt) node[anchor=north east]{\texttt{*+-}}
                
                (tl) circle (2pt) node[anchor=south east]{\texttt{+*-}}
                (A) circle (2pt) node[anchor=north east, outer sep=1mm]{\texttt{++-}}
                (B) circle (2pt) node[anchor=north west, outer sep=1mm]{\texttt{+++}}
                (C) circle (2pt) node[anchor=135, outer sep=2mm]{\texttt{+-+}}
                (AB) circle (2pt) node[anchor=140, outer sep=0.5mm]{\texttt{++*}}
                (BC) circle (2pt) node[anchor=172, outer sep=1mm]{\texttt{+*+}}
                ;

            \filldraw[shift only] 
                (TL) circle (2pt) node[anchor=south]{$w \coloneqq$~\texttt{+**}}
                (LL) circle (2pt) node[anchor=east]{\texttt{**-}};
            
            \draw[dashed] 
                (B) -- (D)
                (BC) -- (DE)
                (C) -- (E)
                (TR) -- (BL);

            \draw[dashed]
                (tl) -- (BL)
                (bl) -- (TL) -- (C) -- (TR) -- (D) -- (BR) -- (E) -- (BL) -- (B) -- (TL)
                (LL) -- (A)
                (TT) -- (C)
                (RR) -- (D)
                (BB) -- (E);

            \node[font=\large] at (-4, 3) {$\Lk_W(w)$};
        \end{tikzpicture}
    \end{center}
    \caption{The subcomplexes $\St_W(w)$ and $\Lk_W(w)$ for $w =$~\texttt{+**}.
    Since $w$ is minimal in $W \setminus V(Q_{\cE, 1})$, its link and star are independent of $W$.} 
    \label{fig:embedding}
\end{figure}
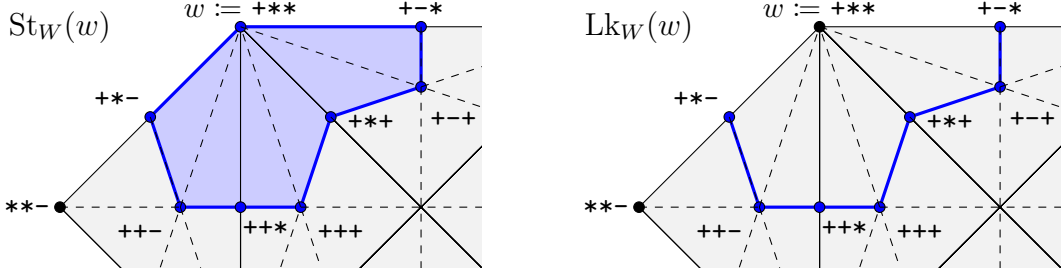 

%% file: main.bbl
\newcommand{\etalchar}[1]{$^{#1}$}
\begin{thebibliography}{VWDP{\etalchar{+}}24}

\bibitem[ABL{\etalchar{+}}22]{Alon_22_private_and_online}
Noga Alon, Mark Bun, Roi Livni, Maryanthe Malliaris, and Shay Moran.
\newblock Private and online learnability are equivalent.
\newblock {\em J. ACM}, 69(4):Art. 28, 34, 2022.

\bibitem[B{\etalchar{+}}02]{bryant2002piecewise}
John~L Bryant et~al.
\newblock Piecewise linear topology.
\newblock {\em Handbook of geometric topology}, pages 219--259, 2002.

\bibitem[BCDK06]{MR2215426}
Hans-J\"urgen Bandelt, Victor Chepoi, Andreas Dress, and Jack Koolen.
\newblock Combinatorics of lopsided sets.
\newblock {\em European J. Combin.}, 27(5):669--689, 2006.

\bibitem[BGH{\etalchar{+}}23]{bun2023stability}
Mark Bun, Marco Gaboardi, Max Hopkins, Russell Impagliazzo, Rex Lei, Toniann
  Pitassi, Satchit Sivakumar, and Jessica Sorrell.
\newblock Stability is stable: Connections between replicability, privacy, and
  adaptive generalization.
\newblock In {\em Proceedings of the 55th Annual ACM Symposium on Theory of
  Computing}, STOC 2023, pages 520--527, 2023.

\bibitem[BGHH25]{BGHH2025stabilitylistreplicabilityagnosticlearners}
Ari Blondal, Shan Gao, Hamed Hatami, and Pooya Hatami.
\newblock Stability and list-replicability for agnostic learners, 2025.
\newblock URL: \url{https://arxiv.org/abs/2501.05333}, \href
  {https://arxiv.org/abs/2501.05333} {\path{arXiv:2501.05333}}.

\bibitem[BHH{\etalchar{+}}25]{blondal2025borsukulamreplicablelearninglargemargin}
Ari Blondal, Hamed Hatami, Pooya Hatami, Chavdar Lalov, and Sivan Tretiak.
\newblock Borsuk-{U}lam and replicable learning of large-margin halfspaces.
\newblock 2025.
\newblock \href {https://arxiv.org/abs/2503.15294} {\path{arXiv:2503.15294}}.

\bibitem[BLM20]{BLM20}
Mark Bun, Roi Livni, and Shay Moran.
\newblock An equivalence between private classification and online prediction.
\newblock In {\em 2020 IEEE 61st Annual Symposium on Foundations of Computer
  Science (FOCS)}, pages 389--402, 2020.

\bibitem[CCH{\etalchar{+}}24]{chase2024dual}
Zachary Chase, Bogdan Chornomaz, Steve Hanneke, Shay Moran, and Amir
  Yehudayoff.
\newblock Dual {VC} dimension obstructs sample compression by embeddings.
\newblock In {\em The Thirty Seventh Annual Conference on Learning Theory},
  pages 923--946. PMLR, 2024.

\bibitem[CCMW22]{chalopin2022unlabeled}
J{\'e}r{\'e}mie Chalopin, Victor Chepoi, Shay Moran, and Manfred~K Warmuth.
\newblock Unlabeled sample compression schemes and corner peelings for ample
  and maximum classes.
\newblock {\em Journal of Computer and System Sciences}, 127:1--28, 2022.

\bibitem[CCMY24]{chase2023local}
Zachary Chase, Bogdan Chornomaz, Shay Moran, and Amir Yehudayoff.
\newblock Local {B}orsuk-{U}lam, stability, and replicability.
\newblock In {\em Proceedings of the 56th Annual ACM Symposium on Theory of
  Computing}, STOC 2024, page 1769–1780. Association for Computing Machinery,
  2024.

\bibitem[CMW25]{chornomaz2025spherical}
Bogdan Chornomaz, Shay Moran, and Tom Waknine.
\newblock Spherical dimension.
\newblock {\em arXiv preprint arXiv:2503.10240}, 2025.

\bibitem[CMY23]{chase2023replicabilitystabilitylearning}
Zachary Chase, Shay Moran, and Amir Yehudayoff.
\newblock {Stability and Replicability in Learning}.
\newblock In {\em 2023 IEEE 64th Annual Symposium on Foundations of Computer
  Science (FOCS)}, pages 2430--2439. IEEE Computer Society, November 2023.

\bibitem[Coo15]{MR3242807}
Michel Coornaert.
\newblock {\em Topological dimension and dynamical systems}.
\newblock Universitext. Springer, Cham, 2015.
\newblock Translated and revised from the 2005 French original.

\bibitem[DPVWV24]{DPVV23}
Peter Dixon, A.~Pavan, Jason Vander~Woude, and N.~V. Vinodchandran.
\newblock List and certificate complexities in replicable learning.
\newblock In {\em Proceedings of the 37th International Conference on Neural
  Information Processing Systems}, NeurIPS '23. Curran Associates Inc., 2024.

\bibitem[EHKS23]{eaton2024replicable}
Eric Eaton, Marcel Hussing, Michael Kearns, and Jessica Sorrell.
\newblock Replicable reinforcement learning.
\newblock In {\em Proceedings of the 37th International Conference on Neural
  Information Processing Systems}, NeurIPS '23. Curran Associates Inc., 2023.

\bibitem[EKK{\etalchar{+}}23]{esfandiari2023replicable}
Hossein Esfandiari, Alkis Kalavasis, Amin Karbasi, Andreas Krause, Vahab
  Mirrokni, and Grigoris Velegkas.
\newblock Replicable bandits.
\newblock In {\em The Eleventh International Conference on Learning
  Representations}, 2023.
\newblock URL: \url{https://openreview.net/forum?id=gcD2UtCGMc2}.

\bibitem[EKM{\etalchar{+}}23]{Esfandiarietal23}
Hossein Esfandiari, Amin Karbasi, Vahab Mirrokni, Grigoris Velegkas, and Felix
  Zhou.
\newblock Replicable clustering.
\newblock In {\em Advances in Neural Information Processing Systems},
  volume~36, pages 39277--39320. Curran Associates, Inc., 2023.
\newblock URL:
  \url{https://proceedings.neurips.cc/paper_files/paper/2023/file/7bc3fe234454107149fa9d44faacaa64-Paper-Conference.pdf}.

\bibitem[Eng78]{MR482697}
Ryszard Engelking.
\newblock {\em Dimension theory}, volume~19 of {\em North-Holland Mathematical
  Library}.
\newblock North-Holland Publishing Co., Amsterdam-Oxford-New York; PWN---Polish
  Scientific Publishers, Warsaw, 1978.
\newblock Translated from the Polish and revised by the author.

\bibitem[Far03]{farley2003finiteness}
Daniel~S Farley.
\newblock Finiteness and {CAT}(0) properties of diagram groups.
\newblock {\em Topology}, 42(5):1065--1082, 2003.

\bibitem[HW41]{dimensiontheory1941}
Witold Hurewicz and Henry Wallman.
\newblock {\em Dimension {T}heory}, volume vol. 4 of {\em Princeton
  Mathematical Series}.
\newblock Princeton University Press, Princeton, NJ, 1941.

\bibitem[KKL{\etalchar{+}}24]{kalavasis2024replicable}
Alkis Kalavasis, Amin Karbasi, Kasper~Green Larsen, Grigoris Velegkas, and
  Felix Zhou.
\newblock Replicable learning of large-margin halfspaces.
\newblock In {\em Proceedings of the 41st International Conference on Machine
  Learning}, ICML'24. JMLR.org, 2024.

\bibitem[KKMV23]{kalavasis2023statistical}
Alkis Kalavasis, Amin Karbasi, Shay Moran, and Grigoris Velegkas.
\newblock Statistical indistinguishability of learning algorithms.
\newblock In {\em Proceedings of the 40th International Conference on Machine
  Learning}, ICML'23. JMLR.org, 2023.

\bibitem[KVYZ23]{karbasi2023replicability}
Amin Karbasi, Grigoris Velegkas, Lin Yang, and Felix Zhou.
\newblock Replicability in reinforcement learning.
\newblock {\em Advances in Neural Information Processing Systems},
  36:74702--74735, 2023.

\bibitem[Law83]{MR683734}
Jim Lawrence.
\newblock Lopsided sets and orthant-intersection by convex sets.
\newblock {\em Pacific J. Math.}, 104(1):155--173, 1983.

\bibitem[Mat03]{MatousekBook}
Ji\v{r}\'{i} Matou\v{s}ek.
\newblock {\em Using the {B}orsuk-{U}lam theorem}.
\newblock Universitext. Springer-Verlag, Berlin, 2003.
\newblock Lectures on topological methods in combinatorics and geometry,
  Written in cooperation with Anders Bj\"orner and G\"unter M. Ziegler.

\bibitem[MM22]{malliaris2022unstable}
Maryanthe Malliaris and Shay Moran.
\newblock The unstable formula theorem revisited via algorithms.
\newblock {\em arXiv preprint arXiv:2212.05050}, 2022.

\bibitem[MSS23]{moran2023bayesian}
Shay Moran, Hilla Schefler, and Jonathan Shafer.
\newblock The bayesian stability zoo.
\newblock {\em Advances in Neural Information Processing Systems},
  36:61725--61746, 2023.

\bibitem[Mun00]{munkres2000topology}
J.R. Munkres.
\newblock {\em Topology}.
\newblock Featured Titles for Topology. Prentice Hall, Incorporated, 2000.

\bibitem[MW16]{moran2016labeled}
Shay Moran and Manfred~K Warmuth.
\newblock Labeled compression schemes for extremal classes.
\newblock In {\em International Conference on Algorithmic Learning Theory},
  pages 34--49. Springer, 2016.

\bibitem[VWDP{\etalchar{+}}24]{vander2024replicability}
Jason Vander~Woude, Peter Dixon, Aduri Pavan, Jamie Radcliffe, and N.V.
  Vinodchandran.
\newblock Replicability in learning: Geometric partitions and {KKM}-{S}perner
  lemma.
\newblock {\em Advances in Neural Information Processing Systems},
  37:78996--79028, 2024.

\end{thebibliography}
